\documentclass[table]{article}
\usepackage[accepted]{icml2026}
\usepackage{tikz} 
\usetikzlibrary{external}
\graphicspath{{figures/tikz\_cached/}}
\usepackage[utf8]{inputenc} 
\usepackage[T1]{fontenc}    
\usepackage{hyperref}       
\usepackage{url}            
\usepackage{booktabs}       
\usepackage{amsfonts}       
\usepackage{nicefrac}       
\usepackage{microtype}      
\usepackage{xcolor}
\usepackage{wrapfig}
\usepackage{enumitem}
\usepackage{diagbox}
\usepackage{pgf}           
\usepackage{placeins} 
\usepackage{amsmath} 
\usepackage{algorithm}
\usepackage{algorithmic}
\usepackage[algo2e, ruled, lined, linesnumbered]{algorithm2e}
\usepackage{hyperref}

\definecolor{goodcolor}{HTML}{FFFFFF} 
\definecolor{lightgreen}{HTML}{90EE90}
\definecolor{neutralcolor}{HTML}{2ECC71} 
\definecolor{badcolor}{HTML}{E74C3C} 
\newcommand{\added}[1]{\textcolor{black}{#1}}

\newenvironment{tightitems}{
  \vspace{-5pt}
  \begin{itemize}[leftmargin=10pt, labelsep=2pt, itemsep=-2pt]
}{
  \end{itemize}
}

\definecolor{mildblue}{RGB}{102, 153, 255} 

\hypersetup{
    urlcolor=mildblue 
}

\hypersetup{
    colorlinks=true,
    citecolor=blue,
    linkcolor=blue,
}
\usepackage{amsmath}
\usepackage{amsthm}
\usepackage{amssymb}
\usepackage{amsfonts}
\usepackage{graphicx}
\usepackage{nomencl}
\usepackage{bbm}
\usepackage{subcaption}
\usepackage{svg}
\usepackage{soul}
\usepackage{mathtools}
\usepackage{booktabs}
\usepackage{multirow}
\usepackage{nomencl}
\usepackage{natbib}
\usepackage[normalem]{ulem}  

\newcommand{\appref}[1]{Appendix~\ref{#1}}

\newcommand{\figref}[1]{Figure~\ref{#1}}
\newcommand{\tabref}[1]{Table~\ref{#1}}

\usepackage{float}
\usepackage{tikz}

\theoremstyle{definition} 

\newtheorem{theorem}{Theorem}
\makeatletter
\@addtoreset{theorem}{section}
\makeatother
\newtheorem{proposition}{Proposition}
\newtheorem{lemma}{Lemma}
\newtheorem{remark}{Remark}

\newcommand{\startpara}[1]{{\noindent{\bf #1.}}} 
\renewcommand{\url}[1]{{\def~{\char126}\sf#1}}

\newcounter{exampcount}
\def\norm#1{\|#1\|}

\def\E{{\mathbb{E}}}

\def\contextset{\Phi}
\def\context{\phi}
\def\rep{b}
\def\repset{\mathcal{B}}

\def\mdp{\mathcal{M}}
\def\prob{\mathbb{P}}

\def\reward{R}
\def\cost{C}

\icmltitlerunning{Adaptive Shielding for Safe Reinforcement Learning under Hidden-Parameter Dynamics Shifts}
\begin{document}
\twocolumn[
  \icmltitle{Adaptive Shielding for Safe Reinforcement Learning \\ under Hidden-Parameter Dynamics Shifts}
  \begin{icmlauthorlist}
    \icmlauthor{Minjae Kwon}{uva}
    \icmlauthor{Tyler Ingebrand}{uta}
    \icmlauthor{Ufuk Topcu}{uta}
    \icmlauthor{Lu Feng}{uva}
    \icmlcorrespondingauthor{Lu Feng}{lu.feng@virginia.edu}
  \end{icmlauthorlist}
  \icmlaffiliation{uva}{The University of Virginia}
  \icmlaffiliation{uta}{The University of Texas at Austin}
  \icmlkeywords{Safe Reinforcement Learning, Runtime Safety, Distribution Shift}
  \vskip 0.3in
]
\printAffiliationsAndNotice{}  

\begin{abstract}
Unseen shifts in environment dynamics, driven by hidden parameters such as friction or gravity, create a challenge for maintaining safety. 
We address this challenge by proposing Adaptive Shielding, a framework for safe reinforcement learning in constrained hidden-parameter Markov decision processes.
A function encoder infers a low-dimensional representation of the underlying dynamics online from transition data, allowing the shield to adapt. To ensure safety during this process, we use a two-layer strategy. First, we introduce safety-regularized optimization that proactively trains the policy away from high-cost regions. Second, the adaptive shielding reactively uses the inferred dynamics to forecast safety risks and applies uncertainty-aware bounds using conformal prediction to filter unsafe actions. We prove that prediction errors in the shielding connect with bounds on the average cost rate. Empirically, across Safe-Gym benchmarks with varying hidden parameters, our approach outperforms baselines on the return-safety trade-off and generalizes reliably to unseen dynamics, while incurring only modest execution-time overhead. Code is available at \href{https://github.com/safe-autonomy-lab/AdaptiveShieldingFE}{adaptive shield}.
\end{abstract}

\section{Introduction}
Robots and other autonomous systems must operate safely in real-world environments where the underlying dynamics can vary due to hidden parameters such as mass distribution, friction, or terrain compliance. These parameters often change across episodes and remain unobserved, introducing safety risks and challenging the generalization capabilities of reinforcement learning (RL) systems~\citep{kirk2023survey, benjamins2023contextualize}. 
Ensuring robust and safe behavior under such uncertainty is essential in domains like autonomous driving and robotic manipulation, where failures can have serious real-world consequences.

Recent work has advanced along two largely separate directions. On one hand, adaptive methods such as hypernetworks, contextual policies, and mixtures-of-experts~\citep{rezaei2023hypernetworks_zero_shot, beukman2023dynamics, 2024gelik_curriculum_crl_moe} enable adaptation to changing dynamics by inferring latent environmental context, but typically do not provide explicit mechanisms for enforcing safety during this process. On the other hand, safe RL frameworks based on constrained Markov decision processes (CMDPs)~\citep{achiam2017constrained, Chen2018, wachi2020safe, yang2020PCPO, yang2022cup} enforce safety by imposing constraints on cumulative costs. Methods such as state augmentation~\citep{li2022saute} and shielding~\citep{Mohammed2017, yang2023safe} enhance safety while maintaining compatibility with a wide range of safe RL algorithms. However, these approaches often assume stationary dynamics and do not perform online adaptation under parameter shift.

To address this gap, we propose an adaptive shielding framework for RL that adapts online to varying hidden parameters while offering safety bounds that are connected to the dynamics model's prediction error. Central to our approach is the use of function encoders~\citep{tyler2024zero_shot, ingebrand2025functionencodersprincipledapproach}, a compact and expressive model class that infers environment dynamics from transition data by projecting them onto neural basis functions. This low-dimensional representation enables reliable online adaptation of both the policy and shield without retraining.

To ensure safe learning and adaptation, our approach combines two complementary mechanisms that operate proactively during training and reactively at execution. First, we introduce a safety-regularized optimization that augments rewards with a cost-sensitive value estimate, encouraging the policy to sample actions leading to long-term unsafe behavior during training. However, this optimization alone cannot ensure safety under distribution shift at execution. To address this, we augment policy execution with an adaptive shield that samples candidate actions from the policy, predicts future states using a function encoder, and applies conformal prediction to construct uncertainty-aware safety margins. Actions that fall below this margin are filtered out, ensuring a conservative runtime safety filter. Empirical evaluations on Safe-Gym benchmarks~\citep{ji2023safety}, including out-of-distribution scenarios with unseen hidden parameters, demonstrate that our method reduces safety violations relative to baselines, while achieving
robust generalization with minimal runtime overhead.

In summary, our main contributions are:
\begin{tightitems}
    \item \textbf{Safety-Regularized Optimization}: We introduce a training-time objective that integrates a cost-sensitive value estimate into reward maximization, optimizing the policy toward low-violation behavior.
    
    \item \textbf{Online Hidden-Parameter Adaptation}: We use function encoders to infer a low-dimensional representation of hidden dynamics from transition data, enabling adaptation of both the policy and the shield without retraining.    
    
    \item \textbf{Adaptive Shield with Uncertainty-Awareness}: We develop an adaptive, uncertainty-aware
shield that filters unsafe actions using conformal prediction, ensuring safety during
execution. We analyze how prediction errors connect with the average cost rate.
\end{tightitems}

\subsection{Related Work}
\startpara{Safe Reinforcement Learning}
Safe RL methods often employ constrained MDP formulations to ensure compliance with safety constraints. Constrained policy optimization (CPO) remains foundational, effectively balancing performance and safety \citep{achiam2017constrained, wachi2020safe}. Further techniques has proposed projection-based updates, recovery policies, 
efficient sample-manipulation, and model-based updates~\citep{zhang2020focops, yang2020PCPO, yang2022cup, gu2024enhancing, as2025actsafe}. Recent zero-violation policy methods in RL aim to minimize safety violations using techniques like genetic cost function search, energy-based action filtering, primal-dual, and primal algorithms \citep{hu2023autocost, zhao2021model, liu2021learning, bai2023achieving, ma2024learn}. However, these approaches often face scalability issues, rely on restrictive assumptions, or are limited to simple environments. Unlike these approaches, we introduce the safety-regularized optimization that can be integrated into the training process of any CMDP-based RL algorithms. Shielding frameworks proactively filter unsafe actions, selectively sampling safe actions~\citep{Mohammed2017, carr2023shielding, yang2023safe}. Recent developments on shielding integrate adaptive conformal prediction into safety frameworks, enhancing uncertainty quantification for safety-critical planning~\citep{lindemann2023safe_planning, sheng2024pomdp_online_shield, sheng2024daptive_conformal, scarbo2025conformal}.
\added{
Control barrier functions (CBFs) offer an alternative certificate-based safety mechanism. However, learning a valid barrier certificate is often difficult under uncertain or varying dynamics, as it typically requires explicit model knowledge or robust bounds on the hidden parameters~\citep{choi2020reinforcement, cheng2023safe, ganai2023iterative, wang2023enforcing, xiao2023bnet}. For details on how this connects to our adaptive shielding mechanism, see Appendix \ref{apx: connection to control theory}.} 
However, unlike existing methods, which are not designed to address varying hidden dynamics, our approach concurrently enhances safety through a safety-regularized optimization and adaptive shielding while adapting to dynamic hidden parameters using function encoders.

\startpara{Contextual or Hidden-Parameter Reinforcement Learning}
Hidden parameters, often termed context, have been studied in recent context-aware reinforcement learning approaches, demonstrating their importance for generalization \citep{benjamins2023contextualize}.
When algorithms are provided with knowledge of the hidden parameters, they are often directly integrated into the model. For example, contextual recurrent state-space models explicitly incorporate known contextual information to enable zero-shot generalization~\citep{prasanna-rlc24a}. Contextualized constrained MDPs further integrate context-awareness into safety-prioritizing curricular learning~\citep{koprulu2025safetyprioritizing}. 
A common approach to handle unknown context information is to infer it from observational history using transformer models \citep{chen2021context}. Hypernetwork-based methods utilize adapter modules to adjust policy networks based on inferred contexts~\citep{beukman2023dynamics}. Mixture-of-experts architectures leverage specialized experts, using energy-based models to handle unknown contexts probabilistically~\citep{2024gelik_curriculum_crl_moe}. However, these works primarily focus on enhancing generalization to varying dynamics without incorporating safety mechanisms during adaptation in contrast to our method.

\startpara{Generalization in Reinforcement Learning}
Generalization in RL, including zero-shot transfer and meta learning, is crucial for robust policy adaptation to varying dynamics.
For example, meta-learning approaches, such as MAML \citep{finn2017model}, allow rapid parameter adaptation from minimal interaction data. Safe meta RL \citep{khattar2023a, guan2024cost, xu2025efficient} extends meta-reinforcement learning to adapt to new tasks while adhering to safety constraints. However, meta-learning approaches involve parameter updates during adaptation, whereas our framework focuses on rapid, online inference of hidden parameters without requiring such updates. Hypernetwork-based zero-shot transfer methods explicitly condition policies on task parameters \citep{rezaei2023hypernetworks_zero_shot}. 
Function encoders, i.e. neural network basis functions, have demonstrated strong zero-shot transfer by using the coefficients of the basis functions as a fully-informative, linear representation of the dynamics ~\citep{tyler2024zero_shot, ingebrand2024zeroshot_ode}.
Single-episode policy transfer and adaptive methods effectively handle environment changes by encoding historical context \citep{yang2019single, chen2022an}. Advanced context encoder designs further improve robustness and fast adaptation capabilities \citep{luo2022context}. While these methods excel at adapting to varying dynamics, they do not address safety constraints during adaptation, leaving agents vulnerable to unsafe actions in unseen environments.

\section{Problem Formulation}
Constrained hidden-parameter MDPs (CHiP-MDPs) model environments with varying transition dynamics, where a cost function is introduced alongside a reward function to address safety constraints. A CHiP-MDP extends the HiP-MDP framework~\citep{konidaris2014hidden} and is defined by the tuple $\mathcal{M}=(S, A, \contextset, T, R, C, \gamma, P_{\contextset})$, where $S$ and $A$ are the state and action spaces, $R: S \times A \times S \rightarrow \mathbb{R}$ is a reward function, $C: S \times A \times S \rightarrow [0, 1]$ is a cost function, and $\gamma \in(0,1)$ is the discount factor. The transition dynamics $T: S \times A \times \contextset \rightarrow S$ depend on a hidden parameter $\context \in \contextset$.
For a specified hidden parameter $\context \in \contextset$, we denote the transition dynamics as $T_\context: S \times A \rightarrow S$.
The prior $P_{\contextset}(\context)$ over the parameter space $\contextset$ represents the distribution of these hidden parameters. We denote the initial state distribution as $\mu_0$.

Since the hidden parameters $\context$ are unknown to the agent, it must infer changes in the environment dynamics from observations. To this end, the agent follows a policy $\pi: S \times \repset \rightarrow A$, where $\repset$ denotes the set of learned representations of the transition dynamics $T_\context$. We denote the resulting representation by $\rep_\context$ for each $\context$. The objective is to maximize expected cumulative discounted reward while satisfying safety constraints in a CHiP-MDP $\mdp$. To formalize this objective, we define the reward action-value function, for a parameter $\context$, initial state $s$, and action $a$ as:
\begin{equation}\label{eqn: q function def}
Q_R^\pi(s, a, \rep_\context) = \mathbb{E}_{\pi, T_\context} \left[ \sum_{t=0}^{\infty} \gamma^t R(s_t, a_t, s_{t+1}) \mid s, a, \context \right].
\end{equation}

The corresponding reward state-value function, which averages $Q_R^\pi$ over actions, is:
\begin{equation}\label{eqn: v function def}
V_R^\pi(s, \rep_\context) = \mathbb{E}_{a \sim \pi(\cdot \mid s, \rep_\context)} \left[ Q_R^\pi(s, a, \rep_\context) \right].    
\end{equation}
Finally, the reward objective is defined as :
\begin{equation}
J_R(\pi) = \mathbb{E}_{\context \sim P_{\context}, s_0 \sim \mu_0(\cdot \mid \context), a_0 \sim \pi(\cdot \mid s_0, \rep_\context)} \left[ Q_R^\pi(s_0, a_0, \rep_\context) \right].    
\end{equation}
\added{
Likewise, the cost objective is defined the same way, replacing the reward function $R$ with the cost function $C$:
$J_C(\pi) = \mathbb{E}_{\context \sim P_{\context}, s_0 \sim \mu_0(\cdot \mid \context), a_0 \sim \pi(\cdot \mid s_0, \rep_\context)} \left[ Q_C^\pi(s_0, a_0, \rep_\context) \right].$
}
The safety constraints aim to minimize the average cost rate. 
To this end, we state our problem below.

\startpara{Problem} Given a CHiP-MDPs $\mathcal{M}=(S, A, \contextset, T, R, C$$,\\ \gamma, P_{\contextset})$ where the transition dynamics $T_\context$ are fully unknown and vary with a hidden parameter $\context$, find an optimal policy $\pi^*$ that maximizes the expected cumulative discounted reward $J_R(\pi^*)$ while satisfying the safety constraints on the average cost rate,
\begin{align}
\xi^{\pi^*}(s, \context) 
&= \lim_{H \to \infty} \frac{1}{H} \mathbb{E}_{\pi^{*}, T_\context} \left[ \sum_{t=0}^{H-1} C(s_t, a_t, s_{t+1}) \mid s, \context \right] \nonumber \\
&\leq \delta,
\end{align}
where $\delta \in (0, 1)$ is a failure probability and initial state is $s$. Note that the transition dynamics depend on the hidden parameter $\context$, but the policy depends on a representation of the hidden parameter, $\rep_\context$, derived from any previously observed transitions by $T_\context$. 

To enforce the safety constraint, we define the cost action-value function $Q_C^\pi$ and cost state-value function $V_C^\pi$ by replacing the reward function $R$ with the cost function $C$ from Equations \ref{eqn: q function def} and \ref{eqn: v function def}. Minimizing the average cost rate can be achieved by minimizing the cost-value function $V^\pi_C$ (see Appendix \ref{apx: avg cost}).

\section{Background}
We introduce key concepts essential for understanding our methods. First, function encoders have demonstrated robust performance in estimating varying underlying dynamics~\citep{tyler2024zero_shot, ingebrand2024zeroshot_ode, ingebrand2025functionencodersprincipledapproach}. Second, conformal prediction provides a rigorous framework for quantifying uncertainty~\citep{vovk2005algorithmic, tibshirani2019conformal, gibbs2024online}.

\startpara{Function Encoder}
A function encoder (FE) offers a compact and computationally efficient framework for representing functions in terms of neural network basis functions. Consider a set of functions $\mathcal{F}=\{f \mid f: \mathcal{X} \rightarrow \mathbb{R}\}$, where $\mathcal{X} \subset \mathbb{R}^n$ is an input space with finite volume. When $\mathcal{F}$ forms a Hilbert space with the inner product $\langle f, g\rangle=\int_{\mathcal{X}} f(x) g(x) d x$, any $f \in \mathcal{F}$ can be expressed using a basis $\left\{g_1, g_2, \ldots, g_k\right\}$ as $f(x)=\sum_{i=1}^k b_i g_i(x)$,
where $b_i$ are unique coefficients. To determine the coefficients, we solve the following least-squares optimization problem: 
\begin{equation}\label{eqn: least square}
(b_1, b_2, \cdots, b_k):=\underset{(b_1, b_2, \cdots, b_k) \in \mathbb{R}^k}{\arg \min }\left\|f-\sum_{j=1}^k b_j g_j\right\|_2^2. 
\end{equation}
For more information on how to train the neural network basis functions, see \citet{ingebrand2025functionencodersprincipledapproach}.

\startpara{Conformal Prediction}  
Conformal Prediction (CP) allows for the construction of prediction intervals (or regions) that are guaranteed to cover the true outcome with a user-specified probability, under minimal assumptions. For exchangeable random variables $\left\{Z_i\right\}_{i=1}^{t+1}$, CP constructs a region satisfying:
$\prob(Z_{t+1} \leq \Gamma_t) \geq 1 - \delta,$
where $\delta \in (0,1)$ is the failure probability, and the threshold $\Gamma_t = Z_{(q)}$ is the $q$-th order statistic of $\left\{Z_1, \ldots, Z_t\right\}$, with $q = \lceil (t+1)(1-\delta) \rceil$. Adaptive Conformal Prediction (ACP) extends this to non-stationary settings by making the threshold learnable. For more information on conformal prediction, see~\citet{shafer2008tutorial, gibbs2021adaptive}.

\section{Approach}
Our approach has three main components. First, we introduce a \emph{safety-regularized optimization (SRO}).
Through this optimization, the policy is trained to converge toward a zero-violation policy. Second, we use a function encoder to represent underlying dynamics $T_\context$, enabling {\em online adaptation}.
Finally, we use this dynamics representation to construct an {\em adaptive shield}. The shield adjusts safe regions by conformal prediction and blocks unsafe actions online. 

\subsection{Safety-Regularized Optimization}\label{subsec:sro}
To promote safe policy learning, we introduce a safety measure, $Q^\pi_{\text{safe}}(s, a, \rep_\context)$, which defines a local cost-sensitivity score for an action $a \sim \pi(\cdot | s, \rep_\context)$, given the learned representation $\rep_\context \in \mathbb{R}^k$. 
The score measures how strongly the expected cumulative cost varies in a neighborhood of $a$, weighted by the policy’s local action density. Higher values of $Q^\pi_{\text{safe}}$ indicate actions with lower long-term costs under policy $\pi$. Since we aim to minimize the cost action-value function $Q^\pi_C$, higher $Q^\pi_{\text{safe}}$ values correspond to lower $Q^\pi_C$ values.
Based on this intuition, we define $Q^\pi_{\text{safe}}(s, a, \rep_\context)$ for an action $a \sim \pi(\cdot |s, \rep_\context)$ as:
\begin{equation}\label{eqn: q_safe}
Q^\pi_{\text{safe}}(s, a, \rep_\context)=-\frac{\int_{B(a, \epsilon) \cap A}\pi(x \mid s, \rep_\context) Q^{\pi}_\cost(s, x, \rep_\context) dx}{V^{\pi}_\cost(s, \rep_\context) + \epsilon}
\end{equation}
where $\epsilon > 0$ is a small constant ensuring numerical stability, and $B(a, \epsilon)$ denotes a small ball of radius $\epsilon$ centered at $a$.

This formulation bounds the value in $(-1, 0]$ by its design. For continuous action spaces, the probability of taking a specific action is always $0$, so we integrate the value over a small interval including that action. For practical implementation, we use a Monte Carlo approximation to the integral by sampling several values around an action $a$ and then aggregating them. See Appendix~\ref{apx: pseudo} for details.

We interpret $Q^\pi_{\text{safe}}(s, a, \rep_\context)$ as a scalar value during optimization (i.e., we apply a stop-gradient to this term). This ensures that the policy updates focus on shifting probability mass toward safer regions. The resulting gradient for the safety objective $J_{\text{safe}}(\pi) = \E[Q^\pi_{\text{safe}}(s, a, \rep_\context)]$ is approximated as:
\begin{equation}
\nabla_\theta J_{\text {safe}}(\pi)
\approx
\mathbb{E}\left[
\nabla_\theta \log \pi_\theta \left(a \mid s, \rep_\context\right)
Q_{\text {safe}}^\pi\left(s, a, \rep_\context\right)
\right].
\label{eq:sro_signal}
\end{equation}

A value of $Q^\pi_{\text{safe}}(s, a, \rep_\context)$ near $0$ implies that $\|\nabla_\theta J_{\text{safe}}(\pi)\|$ is small, so the regularizer contributes little to the policy update. This occurs when either (i) the predicted cumulative cost in a local neighborhood of $a$ is small, or (ii) the policy assigns negligible probability mass to that neighborhood, leading to a small numerator in~\eqref{eqn: q_safe}. 
In contrast, values near $-1$ indicate that regions of high policy probability are associated with large long-term cost, producing a large gradient magnitude that discourages the policy from concentrating on locally high-risk action regions.
See Appendix \ref{apx:design choice} for details on the design choice.

To integrate this safety measure into policy optimization, we define an augmented action-value function, $Q^\pi_{\text{aug}}(s, a, \rep_\context) = Q_\reward^{\pi}(s, a, \rep_\context) + \alpha Q_{\text{safe}}^{\pi}(s, a, \rep_\context)$, where $Q_R^{\pi}(s, a, \rep_\context)$ is the reward action-value function, and
$\alpha \geq 0$ is a hyperparameter balancing safety and reward.
Our {\em safety-regularized optimization} (SRO) updates the policy by maximizing the corresponding
augmented objective:
\begin{equation}
J_\text{aug}(\pi)
=
\mathbb{E}_{\context \sim P_{\context}, \, s_0 \sim \mu_0(\cdot \mid \context), \, a_0 \sim \pi(\cdot \mid s_0, \rep_\context)}
\left[
Q^\pi_{\text{aug}}(s_0, a_0, \rep_\context)
\right].
\end{equation}
A larger $\alpha$ encourages the policy to prioritize safe actions that result in low long-term costs or to select under-explored actions with lower assigned probabilities. Next, we introduce a proposition which justifies this choice of objective.

\begin{proposition}[Reward Consistency within Zero-Violation Policies]\label{thm:prop}
Let $\Pi_{\text{zero-violation}}$ denote the set of zero-violation policies, defined as $\{\pi \mid J_C(\pi) = 0\}$. Then, for any $\alpha \geq 0$, the optimal policy obtained by maximizing the augmented objective function $J_{\text{aug}}(\pi)$ within $\Pi_{\text{zero-violation}}$ is equivalent to the optimal policy obtained by maximizing the standard reward objective $J_\reward(\pi)$ within the same set of policies.        
\end{proposition}

\begin{proof}[Proof Sketch]
For any policy within the zero-violation set, all actions sampled from the policy lead to $Q_C^\pi(s, a, \rep_\context)=0$. By our design of the safety term, this condition implies $Q_{\text {safe}}^\pi(s, a, \rep_\context)=0$. Substituting this into our augmented objective, $J_{\text{aug}}(\pi)$ simplifies to $J_R(\pi)$. 
\end{proof}

Proposition \ref{thm:prop} demonstrates that SRO does not degrade performance unnecessarily
when an agent already behaves safely. Specifically, it guarantees that if we focus only on the set of
policies that satisfy all safety constraints, maximizing the safety-regularized objective is equivalent to maximizing the standard reward objective. SRO preserves the standard monotonic improvement and worst-case constraint bounds of TRPO and CPO, see Appendix~\ref{apx: trpo_sro} for details.

\subsection{Inferring Hidden Parameters Online}\label{subsec:infer hidden}
To infer the underlying dynamics $T_\context$ and predict the next state $s_{t+1}$ based on transition samples and $(s_t, a_t)$, we use a function encoder, denoted by $\hat{f}_{\text{FE}}$. 
Given observed transition samples $\{(s_i, a_i, s_{i+1})\}_{i=1}^{t-1}$ and the current state-action pair $(s_t, a_t)$, the function encoder predicts the next state $\hat{s}_{t+1}$ as:
$\hat{s}_{t+1} = \hat{f}_{\text{FE}}(s_t, a_t) = \sum_{i=1}^k b_i \cdot g_i(s_t, a_t),    $

where $g_i(s_t, a_t)$ are pretrained basis functions, and $\rep_i$ are coefficients derived from a subset of transition samples. These coefficients $\rep_i$ additionally serve as a representation for $T_\context$. Due to the properties of basis functions, these coefficients provide a linear, low-dimensional representation of the underlying dynamics~\citep{tyler2024zero_shot}. We concatenate them with the state to form an augmented input $\left(s_t, \rep_1, \ldots, \rep_k\right)$, which the policy uses as input. As the agent interacts with the environment, collecting new transitions $(s_t, a_t, s_{t+1})$, we refine the coefficients $\rep_i$ by solving Equation \ref{eqn: least square} with updated transition samples. Consequently, the agent receives an online representation of the dynamics. Note that with a fixed number of basis functions $k$, the computation, involving the inverse of a $k \times k$ matrix, remains efficient even for large samples. We denote the coefficients $(\rep_1, \cdots, \rep_k)$ as $\rep_\context$ for the dynamics $T_\context$. We interpret this procedure as a lightweight online system identification, where short transition histories are mapped to a low-dimensional representation of the environment dynamics without retraining the policy.

\subsection{Adaptive Shielding Mechanism}\label{subsec: adaptive shielding}
To ensure safety during policy execution, we propose an adaptive shielding mechanism that dynamically intervenes based on uncertainty in model predictions. This shield wraps any underlying policy $\pi$, adjusting actions to prevent unsafe outcomes. We illustrate the shielding process at timestep $t$. 

We first introduce the necessary settings.
The cost function is defined \added{using an indicator function $\mathbb{I}$} as \( C(s_t, a_t, s_{t+1}) = \mathbb{I}\left\{ \nu(e(s_{t+1}), E_{t+1}) \leq 0 \right\} \),
where \( e: S \to \mathbb{R}^{n_1} \) extracts agent-centric safety features, \( E_{t+1} \in \mathbb{R}^{n_2} \) captures environment features, and \( \nu: \mathbb{R}^{n_1} \times \mathbb{R}^{n_2} \to \mathbb{R} \) is Lipschitz continuous \added{with Lipschitz constant $L_\nu$. We assume the agent-centric safety features change smoothly, i.e., $\left\|e\left(s_{t+1}\right)-e\left(s_t\right)\right\| \leq \Delta_{\max },$
where $\Delta_{\max }$ is a bound on the per-step feature change.}
By the Lipschitz property, the equation \( \| e(s_{t+1}) - e(s_t) \| \leq \Delta_{\max} \) implies:
\begin{equation}
\nu(e(s_{t+1}), E_{t+1}) \geq \nu(e(s_t), E_t) - L_\nu \Delta_{\max}.    
\end{equation}
Thus, if \( \nu(e(s_t), E_t) > L_\nu \Delta_{\max} \), then \( C(s_t, a_t, s_{t+1}) = 0 \) for all \( a_t \in A \). 
Since the value $\nu\left(e\left(s_t\right), E_t\right)$ can be computed at state $s_t$ before selecting action $a_t$ to assess its safety, we call it as the pre-safety indicator.

\setlength{\leftmargini}{10pt}
\setlength{\labelsep}{4pt}
\setlength{\itemsep}{2pt}
\begin{enumerate} 
\item \textbf{Pre-Safety Check}: To minimize intervention, we evaluate the pre-safety indicator:
\begin{equation}
\nu(e(s_t), E_t) > L_\nu \Delta, 
\end{equation}
where $\Delta$ is a predefined value larger than $\Delta_\text{max}$. If this condition is violated, full safety verification is triggered; otherwise, the policy executes directly. This pre-safety check step improves computational efficiency when full safety verification is excessive. 

\item \textbf{Action Generation}: The policy $\pi$ generates $N$ candidate actions \( \{a^{(i)}_t\}_{i=1}^N \) by sampling from its action distribution \( \pi(\cdot \mid s_t, \rep_\context) \), where $\rep_\context$ derived by a subset of transition samples up to time step $t$ explained in Section \ref{subsec:infer hidden}.

\item \textbf{Transition Prediction}: For each candidate action \( a^{(i)}_t \), a function encoder $\hat{f}_{\text{FE}}$ predicts the next state: $\hat{s}_{t+1}^{(i)} = \hat{f}_{\text{FE}}(s_t, a^{(i)}_t)$. Note that any pre-trained forward dynamics model $\hat{f}$ can be used for prediction. However, the function encoder enables inference of varying underlying dynamics and next-state prediction at once.

\item \textbf{Safety Verification}: Using ACP, we compute uncertainty-aware safety margins for each action:
\begin{equation}
\text{SafetyScore}(a^{(i)}_t) = \nu\left(e(\hat{s}_{t+1}^{(i)}), \hat{E}_{t+1}\right) - 2 L_{\nu} \Gamma_t,
\end{equation}
where $\hat{E}_{t+1}$ represents predicted environment features and $\Gamma_t$ is the adaptive conformal prediction bound for $\hat{s}_{t+1}^{(i)}$ and $\hat{E}_{t+1}$ calibrated to target an empirical miscoverage level $\delta$ for the next-state prediction error. Actions are ranked by their safety scores, with positive scores indicating safety compliance. We treat $\Gamma_t$ as an adaptive uncertainty estimate to recent observations, showing a time-varying uncertainty estimate.

\item \textbf{Action Selection}: Define the safe action set at state $s_t$ as \( \hat{A}_{\text{safe}}(s_t) = \{a^{(i)}_t : \text{SafetyScore}(a^{(i)}_t) > 0 \} \) and sampled action set as 
\( \hat{A}_{\text{sample}} = \{a^{(i)}_t\}_{i \in [N]} \)
. The shield executes the following selection rule:
\begin{equation}
    a_t^* = 
    \begin{cases}
        a \sim \mathcal{U}(\text{Top}_k(\hat{A}_{\text{safe}}(s_t))), & \text{if } \hat{A}_{\text{safe}} \neq \emptyset, \\
        \arg\max_{a \in \hat{A}_{\text{sample}}} \text{SafetyScore}(a), & \text{otherwise},
    \end{cases}
\end{equation}
where \( \mathcal{U}(\text{Top}_k(\cdot))\) denotes a uniform distribution over the top \( k \) actions ranked by their safety scores.
\end{enumerate}
When the shield predicts multiple steps $h$ ahead, we repeat the procedure for steps 2, 3, 4, aggregating the safety score over future steps. However, long-term predictions often increase compounding errors and runtime. Thus, we typically use a shorter prediction horizon such as $h=1$ or $h=2$.

The following theorem characterizes how prediction errors in the learned dynamics model connects with a bound on the average cost rate of a policy augmented with the adaptive shield. See \appref{apx: proof} for details. 
\begin{theorem}\label{thm:main}
Given a Constrained Hidden Parameter MDP $\mathcal{M}=(S, A, \contextset, T, R, C, \gamma, P_{\contextset})$ with initial state $s_0 \in S$ and failure probability $\delta \in (0,1)$, assume the prediction error of the learned dynamics model satisfies
$\mathbb{P}(\|s_{t+1} - \hat{s}_{t+1}\| \leq \Gamma_t) \geq 1 - \delta$.
Then, for any policy $\pi$ augmented with the adaptive shield and $\context \sim P_{\contextset}$, there exists $0 \leq \bar{\epsilon} \leq 1$ such that the average cost rate
 constraint: for $\context \sim P_{\contextset}$ and some $0 \leq \bar{\epsilon} \leq 1$,
\begin{align*}
\xi^{\pi}\left(s, \context\right)& =\lim _{H \rightarrow \infty} \frac{1}{H} \mathbb{E}_{\pi, T_\context}\left[\sum_{t=0}^{H-1} C(s_t, a_t, s_{t+1}) \mid s_0 = s, \context \right] \\
& \leq \delta +  \bar{\epsilon}(1 - \delta).    
\end{align*}
\end{theorem}
Given this bound, if safe actions exist at each step, this theorem proves that our algorithm achieves a low average cost rate constraint, governed by the ACP failure probability, i.e., $\xi^{\pi^*}\left(s, \context \right) \leq \delta$.

\begin{proof}[Proof Sketch]
The ACP provides an online-calibrated bound between the predicted state 
$\hat{s}_{t+1} = \hat{f}(s_t, a_t)$ and the true state $s_{t+1}$:
$
\mathbb{P}\left(\|s_{t+1} - \hat{s}_{t+1}\| \leq \Gamma_t\right) \geq 1 - \delta,
$
where $\Gamma_t$ is the confidence region at time $t+1$. 
This bound holds under the assumption that the adaptive conformal procedure maintains calibration of the prediction error, such that $\|s_{t+1} - \hat{s}_{t+1}\| > \Gamma_t$ remains close to the ACP failure probability $\delta$ over time. Since $\nu$ is Lipschitz continuous with constant $L_{\nu}$, we bound the difference in the safety margin between the true and predicted states: 
$
\nu\left(e\left(s_{t+1}\right), E_{t+1}\right) \geq \nu(e\left(\hat{s}_{t+1}\right), \hat{E}_{t+1})-
L_{\nu}\left\|e\left(s_{t+1}\right)-e\left(\hat{s}_{t+1}\right)\right\|-L_{\nu}\left\|E_{t+1}-\hat{E}_{t+1}\right\|.
$
Given that $e$ and $E_{t+1}$ depend on the state prediction, their errors are bounded with high probability: if $e$ is Lipschitz with constant $L_e$, then $\left\|e\left(s_{t+1}\right)-e\left(\hat{s}_{t+1}\right)\right\| \leq L_e \Gamma_t$; similarly, the error in $E_{t+1}$ is bounded by $\Gamma_E \leq L_E \Gamma_t$. For simplicity, we take a uniform bound $\Gamma_t$ when $L_e, L_E \leq 1$. The set of safe actions is defined as:
$\hat{A}_{\mathrm{safe}}\left(s_t\right)=\{a \in A \mid \nu(e(\hat{f}\left(s_t, a\right)), \hat{E}_{t+1})>2 L_{\nu} \Gamma_t\}.
$
If an action is selected from $\hat{A}_{\text {safe }}$, we guarantee $\nu\left(e\left(s_{t+1}\right), E_{t+1}\right)>0$, ensuring a safe state at $t+1$ . The final bound depends on the failure probability of state prediction and the probability of selecting safe actions.
\end{proof}
Guaranteeing $\bar{\epsilon}=0$ requires complete knowledge of dynamics, which is impossible in many settings. Instead, the theorem distinguishes the uncertainty we can reduce through learning from the uncontrollable risks of environments.
\section{Experiments}
We empirically evaluate our approach to assess its safety, generalization, and efficiency across diverse RL tasks. We compare against six safe RL algorithms and analyze three variants of our method: using only safety-regularized optimization (SRO), only the adaptive shield, and their combination. Our experiments are guided by the following research questions:
\begin{tightitems}
    \item \textbf{RQ1:} How does our approach balance safety and task performance during training without being informed of changing hidden parameters?
    \item \textbf{RQ2:} How well does our approach generalize to out-of-distribution test environments by inferring varying hidden parameters online?
    \item \textbf{RQ3:} What is the execution-time overhead introduced by the adaptive shielding mechanism?
\end{tightitems}

\subsection{Experimental Setup}
\startpara{Environments}
We conduct experiments using the Safe-Gym benchmark~\citep{ji2023safety} for safe RL, with two robot types: \emph{Point} and \emph{Car}. Each robot performs four tasks:
(1) \emph{Goal}: navigate to a target while avoiding obstacles;
(2) \emph{Button}: activate a button while avoiding hazards;
(3) \emph{Push}: push an object to a goal under contact constraints;
(4) \emph{Circle}: follow a circular path while staying within safe boundaries.
Robot-task combinations are denoted as {robot}-{task} (e.g., Point-Goal, Car-Circle). Each task includes a safety constraint (e.g., obstacle avoidance or region adherence). Episode-level randomness is introduced by sampling gravity, and four hidden dynamics parameters: damping, mass, inertia, and friction.

\startpara{Baselines}
We compare our approach to six safe RL algorithms:
\begin{tightitems}
\item \textbf{Saute}: A state augmentation technique with safety budgets for almost sure constraint satisfaction, applicable to a wide range of RL algorithms such as PPO or RCPO~\citep{li2022saute}.
\item \textbf{PPO-Lag}: Proximal Policy Optimization with Lagrangian updates for both reward and constraint~\citep{schulman2017ppo, ray2019benchmarking}.
\item \textbf{RCPO}: Reward Constrained Policy Optimization, which uses policy gradients to optimize a reward function penalized by safety violations~\citep{Chen2018}.
\item \textbf{CPO}: Constrained Policy Optimization with joint second-order updates to enforce linearized cost constraints~\citep{achiam2017constrained}.
\item \textbf{CUP}: Constrained Update Projection, a policy optimization method that projects updates to satisfy safety constraints with theoretical guarantees~\citep{yang2022cup}.
\item \textbf{USL}: Unrolling Safety Layer, which re-weights the policy loss for safety and projects unsafe actions into a feasible set at execution~\citep{zhang2023evaluating}.
\end{tightitems}

We compare against oracle baselines that receive the ground-truth hidden parameters $\phi$ as input. This represents a theoretical upper bound for any history-based inference method (e.g., Transformer), which can at best recover $\phi$. In contrast, our approach uses $\hat{f}_{\text{FE}}$ to infer hidden parameters online and is evaluated without this oracle information, demonstrating robustness under limited parameter awareness. \added{We use RCPO without access to the hidden parameters $\phi$ as the base RL algorithm. On top of this, we evaluate RCPO combined with SRO, RCPO combined with Shield, and RCPO combined with both SRO and Shield. For brevity, we refer to these as SRO, Shield, and SRO + Shield, respectively.} We also provide results using PPO-Lag combined with our methods in Appendix \ref{apx: adaptive shielding with ppo}.

\startpara{Hyperparameters}
All methods use the default hyperparameters provided by their respective implementations: Omni-Safe~\citep{omnisafe} for Saute, PPO-Lag, RCPO, CPO, and CUP, and Safe-RL-Kit~\citep{zhang2023evaluating} for USL. When evaluating our approach on top of each base algorithm (e.g., RCPO), we adopt the same hyperparameters as the corresponding baseline to ensure a fair comparison. Each method is trained for 2 million environment steps using 3 random seeds. Each trained policy is evaluated over 100 episodes at test time. 

We set the pre-safety distance to 0.275 and the ACP failure probability to 2\%. The function encoder $\hat{f}_{\text{FE}}$ is pre-trained on 1000 episodes (1000 steps each) collected by a trained PPO policy. The $\hat{f}_{\text{FE}}$ remains fixed during policy training, introducing realistic prediction error that is managed by ACP. All agents are trained under a strict safety constraint, with a cost limit of zero.

For training, environment parameters $\context$ (gravity, damping, mass, inertia, friction) are sampled uniformly from the interval $[0.3, 1.7]$. For out-of-distribution evaluation, the parameters are sampled from the interval $[0.15, 0.3] \cup [1.7, 2.5]$, and the number of obstacles is increased to stress generalization.

\startpara{Metrics}
We evaluate each method using per-episode averages for the following metrics, each capturing a different aspect of performance:
(1) \textit{Return}, measuring task performance as the cumulative reward per episode;
(2) \textit{Cost Rate}, reflecting safety by measuring the frequency of constraint violations per timestep.
(3) \textit{Runtime}, quantifying execution-time efficiency as the wall-clock time per episode; and
(4) \textit{Shielding Rate}, indicating how often the shield is triggered during an episode (i.e., when the pre-safety check fails).

\subsection{Results Analysis}
\begin{figure}[!t]
    \centering
    \includegraphics[width=1\columnwidth]{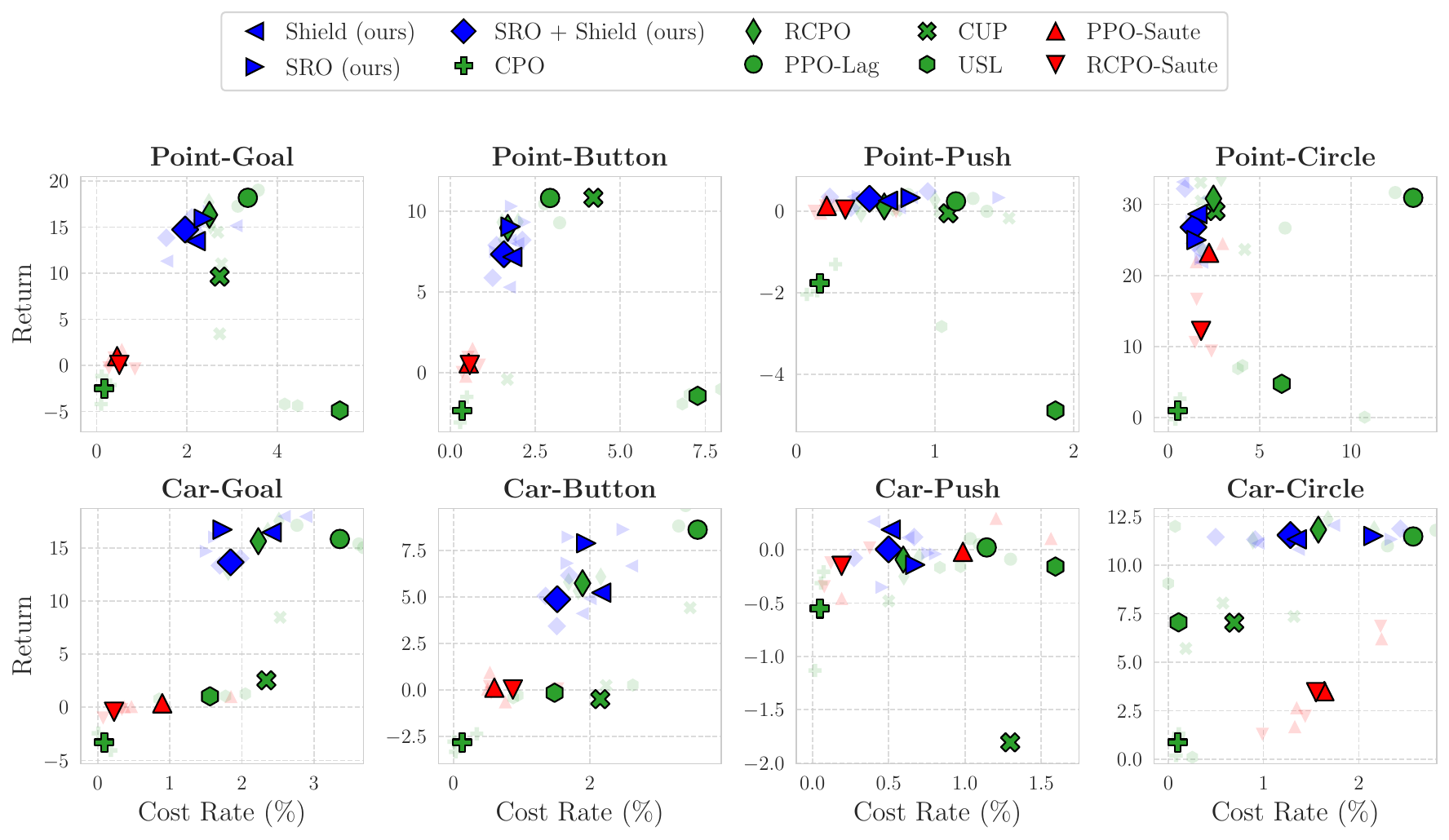}
    \caption{\textbf{Training Dynamics.} Results display the mean reward and cost rate (\%) over the last 20 epochs across seeds. The top-left position is desirable, indicating higher returns with lower cost rates. Solid points represent mean return and cost rate, while transparent points depict individual seed results.}
    \label{fig:rq1}
\end{figure}

\startpara{RQ1: Trade-offs Between Safety and Return}
Figure \ref{fig:rq1} shows the episodic return and cost rate across four tasks during training.
Baseline methods exhibit a range of trade-offs. PPO-Lag tends to achieve high returns but incur higher cost rates. CPO, PPO-Saute, and RCPO-Saute enforce strict safety via a zero-violation constraint, often sacrificing reward learning, which leads to suboptimal policies in multiple tasks. USL, dependent on cost-Q-value estimation, underperforms across all tasks, particularly in environments with randomized obstacle configurations. CUP preserves task-solving performance but frequently violates strict cost limits. RCPO maintains a reasonable balance between safety and return but struggles to meet lower cost thresholds.
These results highlight the challenge of balancing safety and returns, even when hidden parameters are provided as inputs, in the presence of varying hidden parameters. To assess the impact of providing fixed parameters, we compare performance with and without fixed parameters. See \appref{apx: hidden_vs_fixed} for details.

In contrast, our methods (using RCPO as the base RL algorithm) consistently achieve lower cost rates while maintaining competitive returns, demonstrating their ability to balance safety and task performance during training. Variants using only SRO or only adaptive shield also reduce cost violations compared to baselines, but are less effective than the combined method. We observe similar results when using PPO-Lag as the base RL algorithm for our methods (see Appendix~\ref{apx: adaptive shielding with ppo}).

\textit{Takeaway: The combined method (SRO + Shield) achieves a favorable reward-cost trade-off during training and maintains stable performance under unseen variations in environmental parameters.}
\begin{figure}[!t]
    \centering
    \includegraphics[width=1\columnwidth]{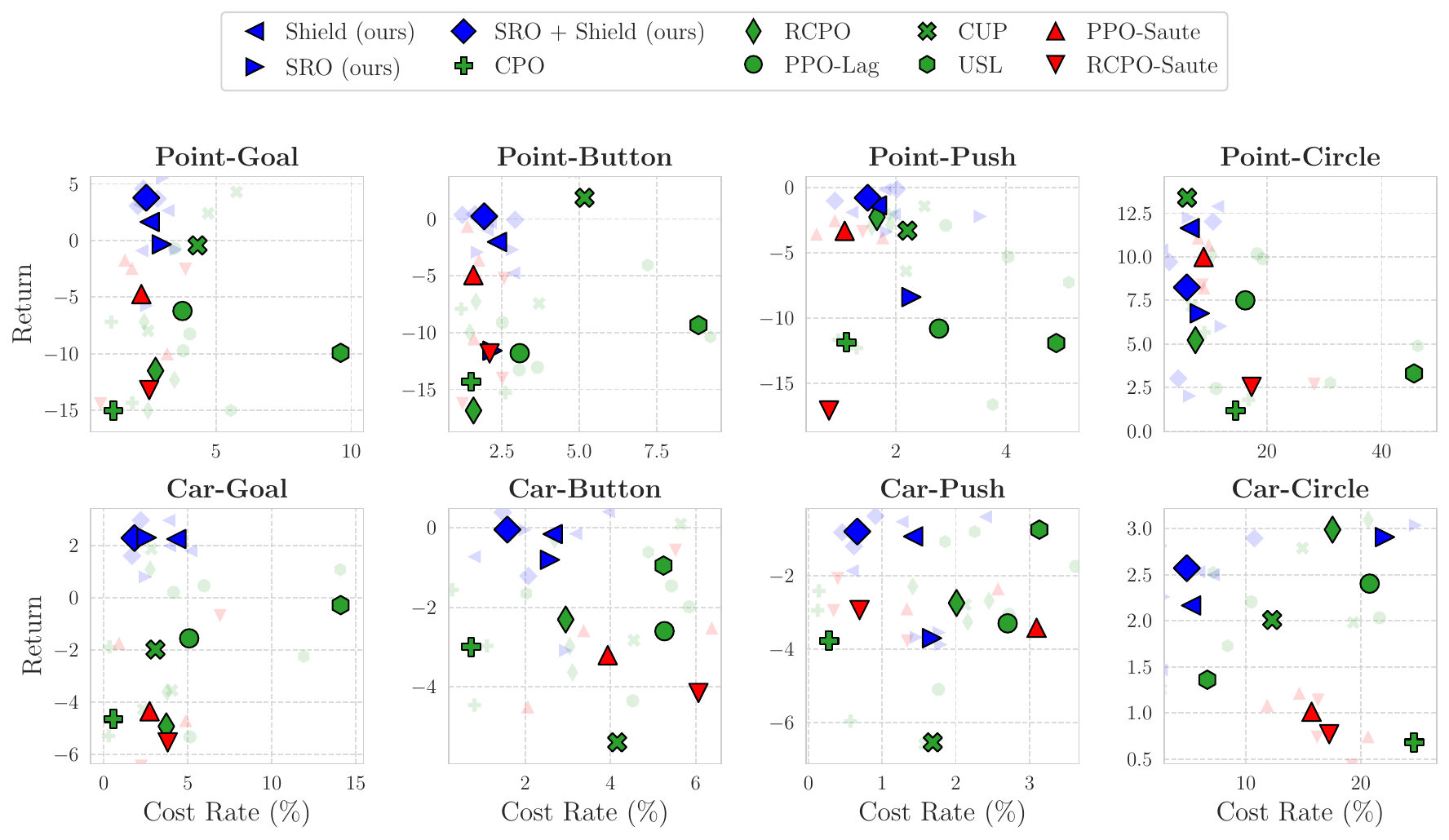}
    \caption{\textbf{OOD Evaluation.} Trade-off between average episodic return and cost rate in out-of-distribution domains. The top-left position is desirable, indicating higher returns with lower cost rates. Solid points represent mean return and cost rate, while transparent points depict individual seed results.}
    \label{fig:rq2_level_2}
\end{figure}

\startpara{RQ2: Generalization to Out-of-Distribution Environments}
\figref{fig:rq2_level_2} illustrates the trade-off between average episodic return and cost rate in out-of-distribution test environments across all tasks. Each marker corresponds to the mean performance across three trained policies (one per seed) for a given method, evaluated separately on Point and Car robots. 

Our combined method (SRO + Shield) consistently appears near the desirable position (high return and low cost rate) across all tasks, indicating robust adaptation to previously unseen dynamics. SRO-only and Shield-only variants also perform well but tend to deviate more from the desirable position. Shield typically remains stable because its logic, which filters unsafe actions based on predicted states and safety measures like proximity to obstacles, holds regardless of OOD conditions. This trend highlights the complementary effect of combining proactive (SRO) and reactive (Shield) safety mechanisms. 

Among the baselines, CPO and Saute maintain low cost rates but sacrifice return, often positioning them outside the desirable region. The remaining algorithms exhibit high variance and sensitivity to dynamics shifts, failing to maintain a consistent safety-performance trade-off across environments.

\textit{Takeaway: The combined method effectively bridges the gap between safety and adaptation, maintaining a favorable return-safety trade-off under out-of-distribution dynamics where non-adaptive baselines fail.}

\startpara{RQ3: Execution-Time Efficiency}
\tabref{tab:rq3_runtime} compares the average runtime per episode between the baseline method (RCPO) and our combined approach (SRO + Shield). Across all tasks and both robot types, the proposed method introduces a modest execution-time overhead relative to the base policy. Additionally, the shield trigger rate remains moderate, indicating that safety shielding is invoked selectively and does not dominate execution time.

\textit{Takeaway: The combined method incurs limited additional computational cost over the base policy while providing consistent runtime safety filtering.}

\setlength{\textfloatsep}{3pt} 
\begin{table}[t]
    \centering
    
    \resizebox{\columnwidth}{!}{
    \begin{tabular}{l|l|c|c|c|c}
        \toprule
        \textbf{Robots} & \textbf{Methods \& Metrics} & \textbf{Goal} & \textbf{Button} & \textbf{Push} & \textbf{Circle} \\
        \midrule
        \multirow{3}{*}{Car}
        & RCPO Runtime (s) & 3.94$\pm$0.03 & 4.18$\pm$0.04 & 4.72$\pm$0.01 & 2.58$\pm$0.02 \\
        & Ours Runtime (s) & 5.05$\pm$0.09 & 5.63$\pm$0.10 & 6.54$\pm$0.13 & 3.04$\pm$0.03 \\
        & Ours Shield Triggers (\%) & 18.10$\pm$2.16 & 27.51$\pm$2.24 & 25.72$\pm$1.37 & 7.32$\pm$2.95 \\
        \midrule
        \multirow{3}{*}{Point}
        & RCPO Runtime (s)  & 3.81$\pm$0.01 & 4.14$\pm$0.03 & 3.48$\pm$0.01 & 2.49$\pm$0.16 \\
        & Ours Runtime (s) & 4.32$\pm$0.66 & 5.60$\pm$0.11 & 5.03$\pm$0.21 & 2.61$\pm$0.04 \\
        & Ours Shield Triggers (\%) & 15.51$\pm$3.82 & 29.07$\pm$2.49 & 7.53$\pm$0.54 & 7.25$\pm$2.84 \\
        \bottomrule
    \end{tabular}
    } 
    \vspace{0.1cm}
    \caption{Runtime (in seconds) and shielding rate (in percent) across tasks for each robot type. “Ours” refers to the combined method using SRO and the adaptive shield.}
    \label{tab:rq3_runtime}
\end{table}

\startpara{Ablation Studies}
We provide additional analyses in the Appendix to assess the robustness of our proposed shielding framework. Key findings include: (1) Compatibility to other Algorithms: We demonstrate that the proposed runtime shielding mechanism improves safety across different base RL algorithms (Appendix~\ref{apx: adaptive shielding with ppo}). This demonstrate that our adaptive shielding does not depend on a specific base algorithm. (2) Dynamics Representation: We show that Function Encoder shows more stable dynamics representations than a Transformer-based alternative, and achieves next-state prediction accuracy comparable to oracle-informed predictors under distribution shift (Appendices~\ref{apx: representation} and \ref{apx:alternative dynamics predictors}). (3) Transfer to Higher-Dimensional Control: We extend evaluation to HalfCheetah-Velocity task with multiple hidden parameters, providing evidence that the proposed approach performs well in more complex task environments (Appendix~\ref{apx: half cheetah}). Additional sensitivity analyses on safety-bonus weight and action-sampling size are shown in Appendix~\ref{apx: sb and ss}. We derive how SRO in theory connects to TRPO or CPO algorithms in Appendix~\ref{apx: trpo_sro}. We provide supporting evidence for our SRO design choice in Appendix~\ref{apx:design choice}. We provide the relationship between our proposed adaptive shield and classical control theory in Appendix~\ref{apx: connection to control theory}.

\section{Conclusion}
We present an adaptive shielding framework for safe and generalizable reinforcement learning that addresses the challenge of maintaining safety while adapting to varying hidden dynamics. By using function encoders, the shield performs inference of hidden dynamics. During training, safety-regularized optimization promotes low-violation behavior in the policy. At execution time, the adaptive shield uses conformal prediction to filter unsafe actions based on uncertainty. Across out-of-distribution Safe-Gym benchmarks, we empirically demonstrate consistent reductions in safety violations relative to baselines while preserving competitive task performance and incurring only modest computational overhead. Our theoretical analysis characterizes how prediction errors in the learned dynamics model connect with bounds on the average cost rate of the shielded policy.

Despite its effectiveness, our approach has several limitations. First, the safety guarantees rely on
assumptions about the structure of the cost function, although these apply to a broad range of practical
scenarios. Second, the method depends on an offline dataset to train the function encoder, which
may limit applicability in settings without prior data. Third, our evaluation has so far been limited
to simulated environments. Future work will aim to address these limitations by relaxing modeling
assumptions, reducing reliance on offline data, and extending evaluations to physical robotic platforms
to assess scalability and real-world applicability.

\section*{Reproducibility Statement}
To ensure the reproducibility of our results, we provided the complete \href{https://github.com/safe-autonomy-lab/AdaptiveShieldingFE}{source code} including a Bash script to execute the entire training and evaluation pipeline. This streamlined approach facilitates use by the research community and provide a reliable baseline for Safe and generalizable RL.


\section*{Broader Impacts}
This work addresses how to maintain safety for RL agents under unseen distribution shifts in environment dynamics caused by hidden parameters (e.g., friction, mass, gravity). We propose an adaptive, uncertainty-aware shielding framework that infers a low-dimensional dynamics representation online and filters unsafe actions at execution time.

\textbf{Benefits.} Runtime shielding can reduce constraint violations without retraining, improving reliability of learning-based controllers in safety-critical domains such as robotics and autonomous systems. The use of uncertainty-aware bounds encourages conservative behavior when predictions are unreliable.

\textbf{Risks.} A primary risk is \emph{false confidence}: probabilistic safety guarantees may be over-trusted outside our modeling assumptions or when the offline data used to train the function encoder poorly covers extreme dynamics. In such cases, safety filtering may degrade or fail.

\textbf{Mitigations.} We recommend treating shielding as one layer in a safety stack, and stress testing under harder distribution shifts, rare events, and adversarial parameter settings. Reporting trigger rates and failure cases can help detect operation outside the method’s reliable guarantees.

\bibliographystyle{icml2026}
\bibliography{references} 
\appendix
\section{Proofs}\label{apx: proof}
This section presents our main theoretical results, including proofs. We first introduce the necessary notations.

\startpara{Notations}
We introduce the notation for dimensions and sets as follows. For $n, m, k_1, k_2 \in \mathbb{N}$, let $n$ denote the state dimension, $m$ the action dimension, $k_1$ the hidden parameter dimension, and $k_2$ the dimension of the function encoder’s learned representation. Note that the dimension $k_2$ of the learned representation may differ from $k_1$, depending on the number of chosen basis functions. The state space is $S \subseteq \mathbb{R}^n$, the action space is $A \subseteq \mathbb{R}^m$, the hidden parameter space is $\contextset \subseteq \mathbb{R}^{k_1}$, and the learned representation space is $\repset \subseteq \mathbb{R}^{k_2}$, where $\repset$ is the coefficient space of basis functions for the function encoder.

\begin{itemize}
    \item $\mdp$: Constrained Hidden Parameter Markov Decision Process.
    \item $s_t \in \mathbb{R}^n$: State at time step $t$.
    \item $\hat{s}_{t} \in \mathbb{R}^n$: Predicted State for time step $t$.
    \item $X_{i,t} ~(X_t) \in \mathbb{R}^2$: Position of the $i$-th obstacle at time step $t$. The index $i$ is omitted when referring to a single obstacle without ambiguity.
    \item $\hat{X}_{i, t} ~(\hat{X}_t)\in \mathbb{R}^2$: Predicted position of the $i$-th obstacle at time step $t$. The index $i$ is omitted when referring to a single obstacle without ambiguity.
    \item $a_t \in \mathbb{R}^m$: Action at time step $t$.
    \item $\hat{f}: S \times A \rightarrow S$: Continuous transition dynamics predictor. 
    \item $\context \in \mathbb{R}^{k_1}$: hidden parameter.
    \item $\rep_\context \in \mathbb{R}^{k_2}$: learned representation to $T_\context$.
    \item $s' \sim T(\cdot \mid s, a, \context)$: Transition dynamics given parameter $\context$. When $s, a, s'$ are unspecified, we denote this by $T_\context$.
    \item $\cost: S \times A \times S \rightarrow [0, 1]$: Cost function bounded in $[0, 1]$.
    \item $\mathcal{S}_{\text{safe}}(s_t, a_t)$: Safe state set, defined as $\{s_t \in \mathcal{S} \mid C(s_t, a_t, s_{t+1}) = 0 \text{ s.t } T(s_{t+1} \mid s_t, a_t, \context)> 0\}$ where $\context$ is a parameter sampled per episode.
    
    \item $Q_R^{\pi}(s, a, \rep_\context)$: State-action value function for reward under policy $\pi$ defined as
    $$
    \mathbb{E}_{\pi, T_\context} \left[ \sum_{t=0}^{\infty} \gamma^t R(s_t, a_t, s_{t+1}) \mid s_0 = s, a_0 = a, \context \right].
    $$
    We aim to maximize this value.
    \item $Q^{\pi}_{C}(s, a, \rep_\context)$: State-action value function for cost under policy $\pi$ defined as
    $$
    \mathbb{E}_{\pi, T_\context} \left[ \sum_{t=0}^{\infty} \gamma^t C(s_t, a_t, s_{t+1}) \mid s_0 = s, a_0 = a, \context \right].
    $$
    We aim to minimize this value.
    \item $Q^\pi_{\text{aug}}(s, a, \rep_\context)$: Safety-regularized state-action value function defined as
    $Q_\reward^{\pi}(s, a, \rep_\context) + \alpha Q_{\text{safe}}^{\pi}(s, a, \rep_\context)$
    where 
    $\alpha$ is a positive constant and
    $$Q^\pi_{\text{safe}}(s, a, \rep_\context)=-\frac{\int_{B(a, \epsilon) \cap A}\pi(x \mid s, \rep_\context) Q^{\pi}_\cost(s, x, \rep_\context)dx}{V^{\pi}_\cost(s, \rep_\context) + \epsilon}$$
    where $\epsilon > 0$ is a small constant ensuring numerical stability, and $B(a, \epsilon)$ denotes a small ball of radius $\epsilon$ centered at $a$.
    We aim to maximize this value, whose maximum value is $0$.

    \item $J_{\text{safe}}(\pi)$: Safety-regularizer for policy $\pi$ defined as $$\mathbb{E}_{\context \sim P_{\context}, s_0 \sim \mu_0(\cdot \mid \context), a_0 \sim \pi(\cdot \mid s_0, \rep_\context)}\left[Q^\pi_{\text{safe}}(s_0, a_0, \rep_\context)\right].$$
    The policy $\pi$ aims to maximize this value.

    \item $J_{\text{aug}}(\pi)$: Safety-regularized objective function for policy $\pi$ defined as $$\mathbb{E}_{\context \sim P_{\context}, s_0 \sim \mu_0(\cdot \mid \context), a_0 \sim \pi(\cdot \mid s_0, \rep_\context)}\left[Q^\pi_{\text{aug}}(s_0, a_0, \rep_\context)\right].$$
    The policy $\pi$ aims to maximize this value.
    
    \item $J_R(\pi)$: Standard reward objective function for policy $\pi$ defined as
    $$\mathbb{E}_{\context \sim P_{\context}, s_0 \sim \mu_0(\cdot \mid \context), a_0 \sim \pi(\cdot \mid s_0, \rep_\context)}\left[Q^\pi_{R}(s_0, a_0, \rep_\context)\right].$$
    The policy $\pi$ aims to maximize this value.

    \item $J_C(\pi)$: Standard cost objective function for policy $\pi$ defined as
    $$\mathbb{E}_{\context \sim P_{\context}, s_0 \sim \mu_0(\cdot \mid \context), a_0 \sim \pi(\cdot \mid s_0, \rep_\context)}\left[Q^\pi_{C}(s_0, a_0, \rep_\context)\right].$$
    The policy $\pi$ aims to minimize this value.
    \item $\Pi_{\text{zero-violation}}$: Set of zero-violation policies defined as $\{ \pi \mid J_C(\pi) = 0\}.$

    \item $\delta \in (0, 1)$: Failure probability.
    \item $\Gamma_t \in \mathbb{R}^{+}$: Adaptive Conformal Prediction (ACP) threshold at time step $t$.
     \item $\xi^{\pi}(s, \context)$: Average cost under policy $\pi$ defined as 
     $$
     \xi^{\pi^*}\left(s, \context\right)=\lim _{H \rightarrow \infty} \frac{1}{H} \mathbb{E}_{\pi^*, T_\context}\left[\sum_{t=0}^{H-1} C(s_t, a_t, s_{t+1}) \mid s, \context \right],$$ starting from an initial state $s$ and parameter $\context$.
    \item $\pi^*$: Optimal policy satisfying constraint on the average cost rate.
\end{itemize}
We restate our proposition and theorems, then provide detailed proofs.
\vspace{0.01cm}

\startpara{Proposition \ref{thm:prop}}
Let $\Pi_{\text{zero-violation}}$ be the set of zero-violation policies. Then, for any $\alpha \geq 0$, the optimal policy obtained by maximizing the safety-regularized objective function $J_{\text{aug}}(\pi)$ within $\Pi_{\text{zero-violation}}$ is equivalent to the optimal policy obtained by maximizing the standard reward objective $J_\reward(\pi)$ within the same set of policies.        

\begin{proof}
By definition, if $\pi \in \Pi_{\text{zero-violation}}$, then for any state $s$, parameter $\context$, and action $a$ with $\pi(a|s, \rep_\context) > 0$, the state-action value function for the cost is zero: $Q^{\pi}_{C}(s, a, \rep_\context) = 0$. Hence, any action sampled around $a$ by the policy $\pi$ will have $Q^{\pi}_{C}(s, a, \rep_\context) = 0$ leading to $$\int_{B(a, \epsilon) \cap A}\pi(x \mid s, \rep_\context) Q^{\pi}_\cost(s, x, \rep_\context)dx=0.$$
This implies that for such policies, the safety term $Q^\pi_{\text{safe}}(s, a, \rep_\context)$ in the regularized objective is a constant value of $0$.

Therefore, for any policy $\pi \in \Pi_{\text{zero-violation}}$, the safety-regularized objective function becomes:
\begin{align*}
  & J_{\text{aug}}(\pi) \\
  & = \mathbb{E}_{\context \sim P_{\context}, s_0 \sim \mu_0(\cdot \mid \context), a_0 \sim \pi(\cdot \mid s_0, \rep_\context)}\left[Q_R^{\pi}(s_0, a_0, \rep_\context) + \alpha \cdot Q^\pi_{\text{safe}}(s_0, a_0, \rep_\context) \right] \\
  & = \mathbb{E}_{\context \sim P_{\context}, s_0 \sim \mu_0(\cdot \mid \context), a_0 \sim \pi(\cdot \mid s_0, \rep_\context)}\left[Q_R^{\pi}(s_0, a_0, \rep_\context) + \alpha \cdot 0\right] \\
  & = J_R(\pi)
\end{align*}
Thus, maximizing $J_{\text{aug}}(\pi)$ within $\Pi_{\text{zero-violation}}$ is equivalent to maximizing $J_R(\pi)$ within the same set. 
\end{proof}

Next, we will prove our main theorem. To prove the main theorem, we first establish the necessary notation and settings. 

We define the safe state set for a given state $s \in S$, action $a \in A$, and hidden parameter $\context \in \contextset$ as:
$
\mathcal{S}_{\mathrm{safe}}(s, a) = \left\{ s' \in S \mid C(s, a, s') = 0  \text { s.t. } T\left(s^{\prime} \mid s, a, \context\right)>0 \right\}.
$
Thus, $\mathcal{S}_{\text{safe}}(s, a)$ contains next states $s'$ where the safety condition is satisfied. 
Throughout the theorem, we address a cost function defined by safe distance
$$\cost(s_t, a_t, s_{t+1}) = 
    \begin{cases} 
    1 & \text{if } \min_{X_{t+1}} \norm{pos(s_{t+1}) - X_{t+1}} \leq d \\ 
    0 & \text{otherwise},
\end{cases}
$$
where $X_{t+1}$ denotes the positions of obstacles at time step $t+1$, and $\text{pos}(s_{t+1})$ represents the agent's position in state $s_{t+1}$. As our theorem applies to any norm satisfying the triangle inequality, we do not specify a particular norm.

We assume the state includes the agent's position, a natural choice for navigation tasks. Generally, the state contains critical information needed to evaluate the cost function. Hence, a function $\text{pos}: S \to \mathbb{R}^2$ is a projection mapping, which is 1-Lipschitz continuous, meaning that 
$$\norm{\text{pos}(s) - \text{pos}(s')} \leq \norm{s - s'}$$ 
for all $s, s' \in S$. 

We assume that the obstacle position $X_{t}$ can be derived from the state $s_{t}$. This assumption is reasonable, as the agent’s state typically includes safety-critical information. For instance, robots in navigation tasks use sensors to detect nearby obstacles, with this information integrated into the agent’s state. This setup applies to all navigation environments in Safety Gymnasium. Formally, we assume a 1-Lipschitz continuous function $\text{sensor}: S \to \mathbb{R}^{2M}$, defined as $\text{sensor}(s_t) = [X_{1,t}, X_{2,t}, \ldots, X_{M,t}]$, where $X_{i,t} \in \mathbb{R}^2$ is the position of the $i$-th obstacle detected by the robot, and $M$ is the number of detected obstacles. When more than $M$ obstacles are present, the sensor typically detects the $M$ closest ones. We adopt a Gaussian policy $\pi$, commonly employed in training RL policies across various algorithms.

\begin{remark}[Validity under Sequential Dependence]
Standard conformal prediction assumes exchangeability of calibration and test data, an assumption that is generally violated in sequential decision-making due to temporal dependence. In our setting, a hidden context parameter $\phi$ is sampled at the beginning of each episode and remains fixed, so that the transition kernel $T_\phi$ is stationary within the episode. However, the marginal distribution of $(s_t, a_t)$ evolves over time as the policy interacts with the environment, and thus the resulting transition tuples are not exchangeable.

To account for this non-exchangeability, we use adaptive conformal prediction (ACP)~\citep{gibbs2021adaptive} as an online calibration mechanism. Our online calibration procedure is designed to control this nonstationarity by restricting calibration and test samples to the same episode and transition model. Specifically, during the first $100$ steps of each episode, we collect a calibration set of transition tuples without applying the ACP region. After this warm-up phase, we use the online-collected calibration data to compute the adaptive conformal threshold $\Gamma_t$ and construct the ACP region.
\end{remark}

Our argument extends to any Lipschitz continuous cost function bounded in $[0, 1]$, with the proof following a similar approach. If the cost function is bounded by a constant $D > 1$, the proof remains valid, but the final bound is scaled by $D$.
\begin{lemma}\label{lemma}
Let $\hat{f}$ be a transition dynamics predictor and $e(a) = \min_{\hat{X}_{t+1}} \|\text{pos}(\hat{f}(s_t, a)) - \hat{X}_{t+1}\|$. Under the adaptive shielding mechanism with sampling size $N$ for each episode with parameter $\context$, one of the following conditions holds for $s_{t+1} \sim T(\cdot \mid s_t, a_t, \context)$:
\begin{center}
\begin{minipage}{0.7\textwidth}
\begin{enumerate}
\item $\prob(s_{t+1} \in S_{\text{safe}}(s_t, a_t)) \geq 1 - \delta$,
\item $\min_{X_{t+1}} \|\text{pos}(s_{t+1}) - X_{t+1}\| \geq \max_{a \in A} e(a) - \epsilon_N - 2\Gamma_t$,
\end{enumerate}    
\end{minipage}
\end{center}
where $\lim_{N \to \infty} \epsilon_N = 0$ and $\Gamma_t$ is the ACP confidence region for the state prediction at time step $t+1$.
\end{lemma}

\begin{proof}
Note that $s_{t+1}$ is safe if
$$\min _{X_{t+1}}\left\|\operatorname{pos}\left(s_{t+1}\right)-X_{t+1}\right\|>d .
$$
The ACP gives us a probabilistic bound on the deviation between the true next state $s_{t+1}$ and the predicted state $\hat{s}_{t+1} = \hat{f}(s_t, a_t)$ :
$$
\prob \left(\|\hat{s}_{t+1} - s_{t+1}\| \leq \Gamma_t\right) \geq 1 - \delta.
$$
Using this bound, we connect the safety of $s_{t+1}$ to the position of the predicted state $\hat{s}_{t+1}$.
By triangle inequality, we have
\begin{align}
\min _{X_{t+1}}& \left\|\operatorname{pos}\left(s_{t+1}\right)-X_{t+1}\right\| 
 \geq \min _{X_{t+1}}\norm{\operatorname{pos}\left(\hat{s}_{t+1}\right)-\hat{X}_{t+1}} \nonumber \\
& -\norm{\hat{X}_{t+1} - X_{t+1}}- 
\left\|\operatorname{pos}\left(s_{t+1}\right)-\operatorname{pos}\left(\hat{s}_{t+1}\right)\right\| .    
\end{align}
Since pos function and sensor function are $1$-Lipschitz, we have
$$
\left\|\operatorname{pos}\left(s_{t+1}\right)-\operatorname{pos}\left(\hat{s}_{t+1}\right)\right\| \leq \left\|s_{t+1}-\hat{s}_{t+1}\right\|\\
$$
and
$$
\norm{X_{t+1}-\hat{X}_{t+1}} \leq \left\|s_{t+1}-\hat{s}_{t+1}\right\|.
$$
Hence, if $\left\|s_{t+1}-\hat{s}_{t+1}\right\| \leq \Gamma_t$ (which happens with probability at least $1-\delta$), then 
$$
\left\|\operatorname{pos}\left(s_{t+1}\right)-\operatorname{pos}\left(\hat{s}_{t+1}\right)\right\| \leq  \Gamma_t \text{ and } \norm{X_{t+1}-\hat{X}_{t+1}} \leq \Gamma_t.
$$
This implies
\begin{equation}\label{eqn:bound}
\min _{X_{t+1}}\left\|\operatorname{pos}\left(s_{t+1}\right)-X_{t+1}\right\| \geq \min _{X_{t+1}}\left\|\operatorname{pos}\left(\hat{s}_{t+1}\right)-X_{t+1}\right\|-2 \Gamma_t .    
\end{equation}
Thus, if $$
h(a_t) = \min _{X_{t+1}}\left\|\operatorname{pos}\left(\hat{s}_{t+1}\right)-\hat{X}_{t+1}\right\| > d + 2\Gamma_t, 
$$
then $\min _{X_{t+1}}\left\|\operatorname{pos}\left(s_{t+1}\right)-X_{t+1}\right\|>d$ whenever $\left\|s_{t+1}-\hat{s}_{t+1}\right\| \leq \Gamma_t$.
Let us define the set of safe actions on the predicted state by
$$
\hat{A}_{\text {safe}}\left(s_t\right)=\left\{a \in A \mid e(a) > d + 2\Gamma_t \right\}.
$$
We now consider two cases based on the feasibility of selecting an action from the set $\hat{A}_{\text{safe}}$.

\textbf{Case 1:} If we can select $a_t \in \hat{A}_{\text {safe}}\left(s_t\right)$ and $\left\|s_{t+1}-\hat{s}_{t+1}\right\| \leq \Gamma_t$, then $s_{t+1}$ is safe by Equation \ref{eqn:bound}. By ACP of our adaptive shielding mechanism, we guarantee 
\[
\prob \left(\|\hat{s}_{t+1} - s_{t+1}\| \leq \Gamma_t\right) \geq 1 - \delta
\]
where $\delta$ is a failure probability of ACP.
Thus, condition 1 holds.

\textbf{Case 2:} If we cannot select $a_t \in \hat{A}_{\text{safe}}(s_t)$, our adaptive shielding mechanism samples $N$ actions $\{a^{(i)}_t\}$ and picks the action $a_t$ such $e(a_t) = \max_{a \in \{a^{(i)}_t\}}e(a)$. Note that $e(a)$ is continuous on $a$ and a Gaussian policy $\pi$ assigns positive probability to any subset of action space $A$. Hence, as sample size $N$ goes to $\infty$, $\max_{a \in A}e(a) - e(a_t) = \epsilon_N$ goes to $0$. Also, by Equation \ref{eqn:bound}, we have 
$$
\|\text{pos}(s_{t+1}) - X_{t+1}\| \geq e(a_t) - 2\Gamma_t = \max_{a \in A}e(a) -\epsilon_N - 2\Gamma_t.
$$
Thus, condition 2 holds.
\end{proof}

To prove the theorem, we recall the function $e(a) = \min_{\hat{X}_{t+1}} \norm{\text{pos}(\hat{f}(s_t, a)) - \hat{X}_{t+1}}$, representing the minimum distance between the predicted state and predicted obstacles. Using this, we define the safe action set for the predicted state as:
$$
\hat{A}_{\text {safe}}\left(s_t\right)=\left\{a \in A \mid e(a) > d + 2 \Gamma_t \right\}.
$$
Lemma \ref{lemma} considers two cases based on whether sampling from $\hat{A}_{\text{safe}}(s_t)$ is feasible. 
To derive a bound for the average cost rate constraint, we analyze both cases by defining $\epsilon_t = \prob(\hat{A}_{\text{safe}}(s_t) = \emptyset)$. We assume that if $\hat{A}_{\text{safe}}(s_t)$ is non-empty, a large sample size $N$ allows sampling an action from this set, as discussed in Lemma \ref{lemma}.

\vspace{0.1cm}
\startpara{Theorem \ref{thm:main}}
Given a Constrained Hidden Parameter MDP $\mathcal{M}=(S, A, \contextset, T, R, C, \gamma, P_{\context})$ with initial state $s_0 \in S$, and failure probability $\delta \in(0,1)$, an optimal policy $\pi^*: S \times \Phi \rightarrow A$, augmented with an adaptive shield, maximizes the expected cumulative discounted return $J_R(\pi^*)$, while satisfying the average cost rate constraint with an initial state $s$:
\begin{align}\label{eqn:main}
\xi^{\pi^*}\left(s, \context\right)
& =\lim _{H \rightarrow \infty} \frac{1}{H} \mathbb{E}_{\pi^{*}, T_\context}\left[\sum_{t=0}^{H-1} C(s_t, a_t, s_{t+1}) \mid s, \context \right] \nonumber \\
& \leq \delta +  \bar{\epsilon}(1 - \delta),    
\end{align}
for some $0 \leq \bar{\epsilon} \leq 1$ and $\context \sim P_{\contextset}$.
\begin{proof}
At each time step $t$, $\hat{A}_{\text{safe}}(s_t)$ is non-empty with probability $1 - \epsilon_t$, allowing us to sample actions with a large sample size $N$. By Lemma \ref{lemma}, this guarantees:
$$
\prob \left(s_{t+1} \in S_{\mathrm{safe}}\left(s_t, a_t\right)\right) \geq 1-\delta
$$

where $s_{t+1} \sim T(\cdot \mid s_t, a_t, \context)$, and the safe state set is defined as:

$$
S_{\mathrm{safe}}\left(s_t, a_t\right)=\left\{s^{\prime}\mid C\left(s_t, a_t, s^{\prime}\right)=0, T_\context(s^{\prime} \mid s_t, a_t)>0\right\}
$$

Thus, when $\hat{A}_{\text{safe}}(s_t)$ is non-empty with probability $1-\epsilon_t$, the cost function satisfies:
$$
C\left(s_t, a_t, s_{t+1}\right)= \begin{cases}1 & \text { with probability at most } \delta \\ 0 & \text { with probability at least } 1-\delta\end{cases}, 
$$
which implies $\prob(C =1) \leq \delta$ and $\prob(C = 0) \geq 1 - \delta$.
In this case,  the expected cost per step is bounded as follows: 
$$\mathbb{E}_{\pi^*}\left[C\left(s_t, a_t,  s_{t+1}\right)\right]  \leq \delta (1 - \epsilon_t).$$
When $\hat{A}_{\text{safe}}(s_t)$ is empty with probability $\epsilon_t$, the expected cost per step is bounded as follows:
$$\mathbb{E}_{\pi^*}\left[C\left(s_t, a_t,  s_{t+1}\right)\right]  \leq \epsilon_t.$$
Combining both cases, the expected cost per step is bounded by:
$$\mathbb{E}_{\pi^*}\left[C\left(s_t, a_t,  s_{t+1}\right)\right]  \leq \delta (1 - \epsilon_t) + \epsilon_t = \delta + \epsilon_t (1 - \delta).$$
By the linearity of expectation, this per-step bound extends to the long-term average cost for a fixed parameter $\context$:
\begin{align}
\xi^{\pi^*}\left(s_0, \context\right)
 & =\lim _{H \rightarrow \infty} \frac{1}{H} \sum_{t=0}^{H-1} \mathbb{E}_{\pi^{*}, T_\context}\left[C\left(s_t, a_t, s_{t+1}\right)\right] \\
 & \leq \limsup _{H \rightarrow \infty} \frac{1}{H} \sum_{t=0}^{H-1}\left(\delta+\epsilon_t(1-\delta)\right)    \\
 & = \delta+\bar{\epsilon}(1-\delta)
\end{align}
where $\bar{\epsilon}=\limsup _{H \rightarrow \infty} \frac{1}{H} \sum_{t=0}^{H-1} \mathbb{E}\left[\epsilon_t\right]$. This satisfies Equation \ref{eqn:main}, completing the proof. Moreover, if safe actions exist at each time step $t$, i.e., $\epsilon_t = \prob(\hat{A}_{\text{safe}} = \emptyset) = 0$, $\bar{\epsilon}$ becomes $0$. Hence, we can bound the equation with a small failure probability $\delta$.
\end{proof}

\section{Compatibility of Safety-Regularized Optimization with TRPO and CPO}\label{apx: trpo_sro}
This section provides a formal extension of the monotonic improvement guarantee of Trust Region Policy Optimization (TRPO) \citep{schulman2015trpo} to our proposed safety-regularized objective, $J_{\text{aug}}(\pi)$. We first recap the foundational theorem of TRPO and then prove that this guarantee directly applies to our augmented objective. Finally, we provide a rigorous analysis of the trade-off between reward and safety that this guarantee implies

This section shows that incorporating SRO into TRPO and CPO preserves their standard monotonic improvement and worst-case constraint bounds. Our analysis follows directly from the original proofs and relies on interpreting SRO as a modification of the scalar advantage signal, rather than as a change to the underlying policy update or trust-region structure.

We formally define the augmented objective as $J_{R}(\pi)=\mathbb{E}\left[Q_{R}^\pi\right], J_{\text{safe}}(\pi) = \E[Q^{\pi}_{\text{safe}}]$ and $J_{\text {aug }}(\pi)=\mathbb{E}[Q_{\text {aug }}^\pi ]$. We restate the policy improvement guarantee of TRPO with SRO. \added{During the optimization, KL constraint plays the same role as in standard TRPO. This only limits how far the updated policy is allowed to move from the current policy in one step, and is independent of whether we optimize the original reward
$J_R$ or our safety-augmented objective $J_{\text{aug}}$.}
\newline{}
\startpara{Theorem 2 \citep{schulman2015trpo}} Let $L_\pi^{\text {aug }}(\tilde{\pi})=J_{\text {aug}}(\pi)+\mathbb{E}_{\context \sim P_{\context} s \sim \rho_\pi, a \sim \tilde{\pi}}\left[A_{\text {aug }}^\pi(s, a, \rep_\context)\right]$. The performance of the new policy $\tilde{\pi}$ is lower-bounded by:
$$
J_{\text {aug }}(\tilde{\pi}) \geq L^\pi_{\text {aug}}(\tilde{\pi})-C_{\text {aug }} \cdot D_{K L}^{\max }(\pi, \tilde{\pi})
$$
where $C_{\text {aug }}=\frac{2 \gamma}{(1-\gamma)^2} \max _{s, a, \context}\left|A_{\text {aug }}^\pi(s, a, \rep_\context)\right|$.

\startpara{Monotonic improvement condition}. If the update $\tilde{\pi}$ satisfies
$$
L_\pi^{\operatorname{aug}}(\tilde{\pi})-C_{\text {aug }} D_{K L}^{\max }(\pi, \tilde{\pi}) \geq J_{\text {aug }}(\pi),
$$
then $J_{\text {aug }}(\tilde{\pi}) \geq J_{\text {aug }}(\pi)$. In other words, any update that sufficiently increases the surrogate while keeping KL small yields non-decreasing augmented performance.

\startpara{Interpretation of the Augmented Objective Trade-off}
Let $\Delta J_R=J_R(\tilde{\pi})-J_R(\pi)$ and $\Delta J_{\text {safe}}=J_{\text {safe}}(\tilde{\pi})-J_{\text {safe}}(\pi)$. Under the monotonic-improvement condition above, for any valid policy update step:
$$
\Delta J_R+\alpha \cdot \Delta J_{\text {safe}} \geq 0
$$
This inequality provides a formal characterization of the trade-off between reward and safety.

\startpara{Bounded Reward Degradation for Safety Improvement}
If an update improves the safety objective ($\Delta J_{\text{safe}}>0$), we can rearrange the inequality to bound the permissible change in reward:
$$
\frac{\Delta J_R}{\Delta J_{\text{safe}}} \geq-\alpha \quad \Longrightarrow \quad \Delta J_R \geq-\alpha \cdot \Delta J_{\text{safe}}
$$
This proves that for a given gain in safety, the reward is guaranteed not to decrease by more than $\alpha$ times that gain. The hyperparameter $\alpha$ thus acts as a maximum acceptable cost in reward for a unit of safety improvement.

\startpara{Bounded Safety Degradation for Reward Improvement}
Conversely, if an update improves the reward objective ($\Delta J_R>0$), we can bound the permissible change in the safety term:
$$
\Delta J_{\text{safe}} \geq-\frac{1}{\alpha} \Delta J_R
$$
This proves that for a given gain in reward, the safety term is guaranteed not to decrease by more than $1 / \alpha$ times that gain. This demonstrates that the algorithm will forgo policy updates that yield high rewards at the expense of excessive safety violations, where the threshold for excessive is explicitly controlled by $\alpha$.

\added{
Now, we also consider safety perspectives of SRO. 
We mainly analyze how SRO affects safety perspective considering theorems from CPO~\citep{achiam2017constrained}.
\newline{}\\
\startpara{Relation to CPO-style constraint guarantees}
CPO's worst-case bound~\citep{achiam2017constrained} depends only on KL constraint and cost advantage $A_C^{\pi_k}$. This is agnostic to how the updated policy $\pi_{k+1}$ is produced. Therefore, the bound from Proposition 2 in \cite{achiam2017constrained}
$$
J_C\left(\pi_{k+1}\right) \leq J_C\left(\pi_k\right)+\frac{\sqrt{2 \delta} \gamma}{(1-\gamma)^2} \epsilon_C^{\pi_{k+1}}
$$
remains the same with $J_\text{aug}$.
The sole role of SRO is to restrict the updated policy $\pi_{k+1}$ so that the quantity $\epsilon_C^{\pi_{k+1}}=\max _s\left|\mathbb{E}_{a \sim \pi_{k+1}} [A_C^{\pi_k}(s, a)]\right|$ remains small.
Because SRO penalizes actions with high estimated long-term cost, $\epsilon_C^{\pi_{k+1}}$ can be related to the safety weight $\alpha$ under the local approximation in~\eqref{eqn: rearange}. Substituting this relation into the CPO bound yields a worst-case cost bound that scales inversely with $\alpha$. Substituting this bound into the CPO inequality provided a worst-case cost guarantee for the SRO-TRPO update.
\newline{}\\
Throughout this section, we assume that the ball $B(a,\epsilon)$ is small enough that local averages approximate point values, i.e., $Q^\pi_{\text{safe}}(s, a) \approx -\frac{Q^\pi_C(s, a)}{V_C^\pi(s) + \epsilon}$.
We begin with introducing necessary Lemma for bound of $A_{\text{aug}}^\pi$.
\begin{lemma}\label{lemma: aug bound}
Let $A_{\mathrm{aug}}^\pi(s, a)=A_R^\pi(s, a)+\alpha A_{\mathrm{safe}}^\pi(s, a),$
where $A_R^\pi(s, a)=Q_R^\pi(s, a)-V_R^\pi(s),$
$$
A_{\text {safe }}^\pi(s, a)=Q_{\text {safe }}^\pi(s, a)-V_{\text {safe }}^\pi(s),
$$ and $$V_{\text {safe }}^\pi(s)=\mathbb{E}_{a^{\prime} \sim \pi(\cdot \mid s)}\left[Q_{\text {safe}}^\pi(s, a^{\prime}\right)].
$$
Then, $Q_{\text {safe }}^\pi(s, a) \in(-1,0]$ for all $(s, a)$, and
$
\left|A_{\mathrm{aug}}^\pi(s, a)\right| \leq\left|A_R^\pi(s, a)\right|+\alpha$
\end{lemma}
\begin{proof} By design, $Q_{\text {safe }}^\pi(s, a) \in(-1,0]$. Consequently, the value function $V_{\text {safe }}^\pi(s)$, being an expectation of $Q$, is also in ( $-1,0$ ]. The advantage is defined as $A_{\text {safe }}^\pi(s, a)=Q_{\text {safe }}^\pi(s, a)- V_{\text {safe }}^\pi(s)$. The maximum possible value is $0-(-1)=1$ (when $Q=0, V=-1$ ). The minimum possible value is $-1-0=-1$ (when $Q=-1, V=0$). Thus, $\left|A_{\text{safe}}^\pi(s, a)\right| \leq 1$. By triangular inequality, we have
$
\left|A_{\mathrm{aug}}^\pi(s, a)\right| = |A^\pi_R(s, a) + \alpha A_\text{safe}^\pi(s, a)| \leq\left|A_R^\pi(s, a)\right|+\alpha.
$
\end{proof} 
We will split $A_{\text{aug}}^\pi$ into two parts $A^\pi_R$ and $A^\pi_C.$
Since the ball $B(a,\epsilon)$ is small enough that local averages approximate point values, our construction reduces to
\begin{align}
A_{\text{safe}}^{\pi_k}(s, a) & =Q_{\text{safe}}^{\pi_k}(s, a)-V_{\text{safe}}^{\pi_k}(s) \approx -\frac{Q_C^{\pi_k}(s, a)-V_C^{\pi_k}(s)}{V_C^{\pi_k}(s)+\varepsilon} \nonumber \\
& =-\frac{A_C^{\pi_k}(s, a)}{V_C^{\pi_k}(s)+\varepsilon}.    
\end{align}
Hence, we have 
\begin{align}\label{eqn: rearange}
A_{\text{aug}}^{\pi_k}(s,a) & = A_R^{\pi_k}(s, a)+\alpha A_{\mathrm{safe}}^{\pi_k}(s, a) 
 \nonumber \\ 
 & \approx A_R^{\pi_k}(s,a) - \lambda(s) A_C^{\pi_k}(s,a),
\end{align}
where 
\begin{equation}\label{eqn: lambda}
\lambda(s) = \frac{\alpha}{V_C^{\pi_k}(s) + \epsilon}.    
\end{equation}
We assume that the reward and cost are bounded by $R_{\text{max}}$ and $C_{\text{max}}$, respectively.
Then, for all $s, a$, we have 
\begin{align}\label{eqn: cost value bound}
& |A_R^{\pi_k}(s, a)| \leq \frac{R_{\text{max}}}{1-\gamma}, |A_C^{\pi_k}(s, a)| \leq \frac{C_{\text{max}}}{1-\gamma} \\
& V_C^{\pi_k}(s) \leq \frac{C_{\text{max}}}{1-\gamma}.    
\end{align}
This implies that, for all $s$, we have
\begin{equation}\label{eqn: lambda bound}
\lambda(s)=\frac{\alpha}{V_C^{\pi_k}(s)+\epsilon} \geq \frac{\alpha}{\frac{C_{\max }}{1-\gamma}+\epsilon} = \frac{\alpha(1-\gamma)}{C_{\text{max}} + \epsilon(1 -\gamma)}.
\end{equation}
Moreover, by rearranging Equation \ref{eqn: rearange}, we have
$$
 A_C^{\pi_k}(s,a) \approx \frac{A_R^{\pi_k}(s, a) - A_{\text{aug}}^{\pi_k}(s,a)}{\lambda(s)}.
$$
Combining Equations \ref{eqn: cost value bound} and \ref{eqn: lambda bound}, and Lemma \ref{lemma: aug bound}, we have
\begin{align}
|A_C^{\pi_k}(s,a)| 
 & \leq \frac{1}{\lambda(s)} \left(|A_R^{\pi_k}(s, a)| + |A_{\text{aug}}^{\pi_k}(s,a)| \right)\\
 & \leq \frac{1}{\lambda(s)} \left(\frac{R_{\text{max}}}{1-\gamma} + \frac{R_{\text{max}}}{1-\gamma} + \alpha \right) \\
 & \leq \frac{C_{\text{max}} + \epsilon(1 -\gamma)}{\alpha(1-\gamma)} \left(\frac{2R_{\text{max}}}{1-\gamma}  + \alpha \right) \\
 & = \frac{2R_{\text{max}}(C_{\text{max}} + \epsilon(1 -\gamma))}{\alpha(1-\gamma)^2}   +  \frac{C_{\text{max}}}{(1-\gamma)} + \epsilon.
\end{align}
Thus, we have
\begin{equation}\label{eqn: A_C bound}
|A_C^{\pi_k}(s,a)|  \leq O\left(\frac{1}{\alpha}\right) + \frac{C_{\text{max}}}{(1-\gamma)} + \epsilon. 
\end{equation}
By the definition of $\epsilon_C^{\pi_{k+1}}=\max _s\left|\mathbb{E}_{a \sim \pi_{k+1}} [A_C^{\pi_k}(s, a)]\right|$, and Equation \ref{eqn: A_C bound}, CPO worst-case bound (Proposition 2 in \cite{achiam2017constrained})
\begin{equation}
J_C(\pi_{k+1})
\le
J_C(\pi_k)
+
\frac{\sqrt{2\delta}\gamma}{(1-\gamma)^2} \varepsilon_C^{\pi_{k+1}}    
\end{equation}
becomes
\begin{equation}
J_C(\pi_{k+1})
\le
J_C(\pi_k)
+
O\!\left(\frac{\sqrt{\delta}}{\alpha}\right) + \frac{C_{\text{max}} \sqrt{2\delta\gamma}}{(1-\gamma)^3} + \epsilon \frac{\sqrt{2\delta}\gamma}{(1-\gamma)^2} ,    
\end{equation}
showing that stronger safety regularization (larger $\alpha$) tightens the worst-case constraint-violation bound for a fixed KL radius $\delta$.
}
\newpage
\section{Pseudocode for the safety-regularized optimization for actor-critic} \label{apx: pseudo}
For clarity in this section only, we denote the parameters for policy, reward value critic, cost value critic, and cost Q-critic by $\theta$, $\phi$, $\psi$, and $\omega$, respectively. We use $sg(\cdot)$ to indicate a stop-gradient operation.
\begin{algorithm}[h]
\caption{Actor-Critic with Safety-Regularized Objective}
\label{alg:safety_ac}
\begin{algorithmic}[1]
\STATE \textbf{Initialize:} Policy $\pi_\theta$, Reward Critic $V^\phi_R$, Cost Value Critic $V^\psi_C$, Cost Q-Critic $Q^\omega_C$.
\STATE \textbf{Hyperparameters:} Safety bonus coefficient $\alpha$, KL constraint $\delta_{KL}$.
\FOR{each training epoch}
    \STATE Collect trajectories $\mathcal{D}$ using current policy $\pi_\theta$.
    \STATE Compute targets $V_R^{\text{targ}}, V_C^{\text{targ}}$ and advantages $A_R, A_C$ using GAE.
    \STATE Store observations, actions, log-probs, and targets in buffer.
    \FOR{$k = 1$ to $K$ update iterations}
        \STATE Sample mini-batch $\mathcal{B} = \{(s, a, \log \pi_{old}, V^{\text{targ}}, A)\}$ from $\mathcal{D}$.
        \STATE \textbf{// 1. Update Reward Critic}
        \STATE Minimize $\mathcal{L}_{V_R} = \mathbb{E}_{\mathcal{B}} \left[ (V^\phi_R(s) - V_R^{\text{targ}})^2 \right]$.
        \STATE \textbf{// 2. Update Cost Critics ($V_C$ and $Q_C$)}
        \STATE Compute $V_{stop\_grad} = sg(V^\psi_C(s))$.
        \STATE Minimize $\mathcal{L}_{V_C} = \mathbb{E}_{\mathcal{B}} \left[ (V_{stop\_grad} - V_C^{\text{targ}})^2 \right]$.
        \STATE Minimize $\mathcal{L}_{Q_C} = \mathbb{E}_{\mathcal{B}} \left[ ((Q^\omega_C(s, a) - V_{stop\_grad}) - A_C)^2 \right]$.
        \STATE \textbf{// 3. Compute Safety Estimates}
        \STATE Sample $N$ noise vectors $\epsilon \sim \mathcal{N}(0, \sigma)$ to generate local actions $a' = a + \epsilon$.
        \STATE Estimate local Cost Q-values $Q_{approx} = \frac{1}{N} \sum_{a'} \pi(a'\mid s)Q^\omega_C(s, a')$.
        \STATE Compute safety regularizer $Q_{\text{safe}} = \frac{Q_{approxi}}{sg(V^\psi_C) + \epsilon} \approx \text{Equation}~\ref{eqn: q_safe}$, .
        \STATE \textbf{// 4. Update Actor}
        \STATE Compute surrogate advantage: 
        \begin{equation*}
            A_{\text{aug}} = A_R + \alpha \cdot Q_{\text{safe}}
        \end{equation*}
        \STATE Update $\theta$ by maximizing policy objective $\mathcal{L}_\pi$ using $A_{\text{aug}}$.
    \ENDFOR
\ENDFOR
\end{algorithmic}
\end{algorithm}

\section{Connection to Control Theory}\label{apx: connection to control theory}
In this section, we introduce a brief overview of barrier certificate approaches in control theory and then relate these ideas to our shielding mechanism.
\newline{}\\
Recent advances in safe control and safe reinforcement learning suggest using Control Barrier Functions (CBFs) or barrier-like certificates to establish forward invariance of a safe set. Classical and neural CBF methods construct a differentiable barrier function \(h(s)\) whose evolution satisfies a discrete- or continuous-time invariance condition, and use this certificate to guarantee safety under learned or partially known dynamics. This paradigm has been widely applied to safe RL, including disturbance-observer-based barrier methods \citep{cheng2023safe}, reachability-based approximations \citep{ganai2023iterative}, soft-barrier formulations for stochastic environments \citep{wang2023enforcing}, neural CBFs integrated directly into RL \citep{xiao2023bnet}, and the joint use of CBFs and control Lyapunov functions (CLFs) for enhanced stability and safety under model uncertainty \citep{choi2020reinforcement}.
These approaches learn or optimize a barrier function jointly with the policy or the dynamics model, and safety depends on the existence of a valid barrier certificate, which is often difficult to find.
\paragraph{Control Barrier Function}
A function $\alpha : [0, a) \to [0, \infty)$ is called a \emph{class-$\mathcal{K}$ function}
if it is continuous, strictly increasing, and satisfies $\alpha(0) = 0$.
In classical nonlinear control, forward invariance of a safe set $\mathcal{S}_{\text{safe}} \subseteq \mathcal{S}$ is certified through a
Control Barrier Function (CBF) $h:\mathcal{S}\rightarrow\mathbb{R}$ that satisfies the differential constraint
\begin{equation}
\dot h(s) + \alpha(h(s)) \ge 0 ,
\label{eq:cbf}
\end{equation}
for an extended class-$\mathcal{K}$ function $\alpha$. 
If \eqref{eq:cbf} holds for all admissible controls, then trajectories starting in $\mathcal{S}_{\text{safe}} = \{s \mid h(s)\ge 0\}$ remain in that set for all future times, guaranteeing forward invariance.
The shielding guarantee proposed in Theorem~1 is closely related to the notion of \emph{forward invariance} in control theory.
\paragraph{Relation to our Shield.}
Our adaptive shield provides an analogous guarantee in the discrete-time and data-driven setting.
Rather than enforcing the differential condition \eqref{eq:cbf},
we bound the change in a safety function $\nu(e(s_t),E_t)$ between successive steps using its Lipschitz continuity,
\begin{align}
\nu(e(s_{t+1}),E_{t+1}) 
& \ge \nu(e(s_t),E_t) - L_\nu \,\|e(s_{t+1})-e(s_t)\| \nonumber \\
& \ge \nu(e(s_t),E_t) - L_\nu \Delta_{\max}.
\label{eq:lipschitz_bound}    
\end{align}
Thus, whenever $\nu(e(s_t),E_t)>L_\nu\Delta_{\max}$, the next state $s_{t+1}$ remains within the safe region.
This is a discrete-time forward-invariance condition derived from the structure of the cost function $C(s_t,a_t,s_{t+1})$ and the Lipschitz property of $\nu$.
The adaptive conformal bound $\Gamma_t$ introduced in the shield
plays the role of a stochastic disturbance margin, producing a probabilistic forward invariance guarantee. 
\newline{}\\
While neural and classical CBF methods construct a barrier certificate $h(s)$ satisfying the invariance condition Equation \ref{eq:cbf} either through analytical dynamics or by learning $h$ jointly with a dynamics model, our approach leverage the structure of the cost function such as Lipschitz property, and conformal prediction bound without learning a barrier function. 

\section{Ablation Study on Safety Bonus and Sampling Size}\label{apx: sb and ss}
We evaluate the hyperparameter sensitivity of our method combined with RCPO, focusing on safety bonus $\alpha$ and sampling size $s$.

\startpara{Varying Safety Bonus $\alpha$ with RCPO}
To assess the sensitivity of the safety bonus, we vary the safety bonus $\alpha$ across $\{0.05, 0.1, 0.5, 1.0\}$. We observe that performance with SRO often improves both reward and safety. This is mainly because SRO encourages the policy to select safe actions while also exploring under-explored actions. However, SRO does not consistently enhance both reward and safety; instead, it frequently improves either reward or cost. These findings align with the theoretical results presented in Appendix~\ref{apx: trpo_sro}.
\begin{table}[!htbp]
\centering
\resizebox{\columnwidth}{!}{%
  \begin{tabular}{lcccccccc}
  \hline
  \diagbox{Algo.}{Env.} & \multicolumn{2}{c}{Point-Goal} & \multicolumn{2}{c}{Point-Button} & \multicolumn{2}{c}{Point-Push} & \multicolumn{2}{c}{Point-Circle} \\
  \hline
  & \textbf{R} $\uparrow$ & \textbf{C}(\%) $\downarrow$ & \textbf{R} $\uparrow$ & \textbf{C}(\%) $\downarrow$ & \textbf{R} $\uparrow$ & \textbf{C}(\%) $\downarrow$ & \textbf{R} $\uparrow$ & \textbf{C}(\%) $\downarrow$ \\
  \hline
  RCPO & 16.35±1.14 & 2.50±0.02 & 8.97±0.56 & 1.69±0.23 & 0.12±0.15 & 0.63±0.28 & \textbf{30.78±5.55} & 2.47±0.51 \\
SRO + RCPO ($\alpha$=0.05) & 17.49±1.82 & 2.89±0.52 & 10.38±1.89 & 1.76±0.24 & \textbf{0.33±0.16} & 0.88±0.48 & 24.33±4.58 & 1.55±0.03 \\
SRO + RCPO ($\alpha$=0.1) & 17.56±1.35 & 2.63±0.25 & 10.28±2.00 & \textbf{1.64±0.23} & 0.32±0.15 & 0.65±0.37 & 25.95±4.37 & \textbf{0.98±0.02} \\
SRO + RCPO ($\alpha$=0.5) & \textbf{17.58±0.56} & 2.49±0.34 & \textbf{11.26±1.28} & 1.82±0.31 & 0.33±0.12 & \textbf{0.61±0.10} & 25.18±4.50 & 1.28±0.29 \\
SRO + RCPO ($\alpha$=1.0) & 15.93±0.78 & \textbf{2.37±0.11} & 9.05±1.43 & 1.77±0.43 & 0.33±0.04 & 0.82±0.56 & 24.97±3.74 & 1.55±0.34 \\
  \hline
  \end{tabular}%
}
\end{table}

\begin{table}[!htbp]
\centering
\resizebox{\columnwidth}{!}{%
  \begin{tabular}{lcccccccc}
  \hline
  \diagbox{Algo.}{Env.} & \multicolumn{2}{c}{Car-Goal} & \multicolumn{2}{c}{Car-Button} & \multicolumn{2}{c}{Car-Push} & \multicolumn{2}{c}{Car-Circle} \\
  \hline
  & \textbf{R} $\uparrow$ & \textbf{C}(\%) $\downarrow$ & \textbf{R} $\uparrow$ & \textbf{C}(\%) $\downarrow$ & \textbf{R} $\uparrow$ & \textbf{C}(\%) $\downarrow$ & \textbf{R} $\uparrow$ & \textbf{C}(\%) $\downarrow$ \\
  \hline
  RCPO & 15.64±2.27 & 2.23±0.33 & 5.74±0.34 & 1.89±0.22 & -0.09±0.13 & 0.60±0.09 & \textbf{11.83±0.49} & 1.58±0.57 \\
SRO + RCPO ($\alpha$=0.05) & \textbf{18.54±0.65} & 2.13±0.14 & 7.01±0.51 & \textbf{1.80±0.26} & -0.10±0.17 & 0.88±0.25 & 11.29±0.28 & \textbf{1.29±0.19} \\
SRO + RCPO ($\alpha$=0.1) & 15.88±1.12 & 2.31±0.66 & \textbf{7.68±1.21} & 1.85±0.19 & \textbf{0.06±0.09} & 0.68±0.41 & 11.47±0.53 & 1.51±0.61 \\
SRO + RCPO ($\alpha$=0.5) & 17.33±0.58 & 2.55±0.50 & 6.71±1.11 & 1.88±0.46 & -0.05±0.12 & \textbf{0.59±0.39} & 11.19±0.25 & 1.52±0.47 \\
SRO + RCPO ($\alpha$=1.0) & 16.73±2.46 & \textbf{1.73±0.29} & 7.89±0.95 & 1.95±0.47 & -0.14±0.18 & 0.67±0.19 & 11.50±0.16 & 2.15±0.50 \\
  \hline
  \end{tabular}%
}
\caption{Ablation Study on the Varying Effects of Safety Bonus $\alpha$ on Safety and Performance. Best performances (highest return and lowest cost rate) are highlighted in bold.}
\label{tab:safety_bonus_rcpo}
\end{table}
\FloatBarrier

\startpara{Varying Sampling Numbers $s$ with RCPO}
We evaluate the impact of sampling size, varying it across $\{5, 10, 20, 50\}$ by fixing safety bonus $\alpha =1.0$. Table~\ref{tab:sampling_abl_rcpo} demonstrates that sampling size influences performance. Sampling numbers exhibit no consistent pattern due to randomness in the sampling procedure. This arises primarily from high prediction errors, which often lead to incorrect action sampling, even within conformal prediction boundaries. For example, a large error widens the conformal interval range, causing the shield to include numerous sampled actions to meet the probabilistic guarantee. However, even when selecting actions based on safety scores, these high errors may inaccurately represent safe actions. 
\begin{table}[!htbp]
\centering
\resizebox{\columnwidth}{!}{%
  \begin{tabular}{lcccccccc}
  \hline
  \diagbox{Algo.}{Env.} & \multicolumn{2}{c}{Point-Goal} & \multicolumn{2}{c}{Point-Button} & \multicolumn{2}{c}{Point-Push} & \multicolumn{2}{c}{Point-Circle} \\
  \hline
  & \textbf{R} $\uparrow$ & \textbf{C}(\%) $\downarrow$ & \textbf{R} $\uparrow$ & \textbf{C}(\%) $\downarrow$ & \textbf{R} $\uparrow$ & \textbf{C}(\%) $\downarrow$ & \textbf{R} $\uparrow$ & \textbf{C}(\%) $\downarrow$ \\
  \hline
  RCPO & 16.35±1.14 & 2.50±0.02 & \textbf{8.97±0.56} & 1.69±0.23 & 0.12±0.15 & 0.63±0.28 & \textbf{30.78±5.55} & 2.47±0.51 \\
Shield + SRO (s=5) & 12.83±1.91 & 2.27±0.60 & 8.68±2.51 & 2.43±0.39 & 0.13±0.27 & \textbf{0.51±0.27} & 26.76±4.71 & \textbf{1.38±0.45} \\
Shield + SRO (s=10) & \textbf{16.40±0.73} & 2.29±0.21 & 8.22±3.31 & 1.98±0.72 & \textbf{0.36±0.40} & 0.65±0.04 & 29.97±0.43 & 2.16±0.43 \\
Shield + SRO (s=20) & 14.74±1.18 & \textbf{1.96±0.37} & 7.43±1.55 & 2.19±0.45 & 0.03±0.18 & 0.90±0.53 & 26.76±5.40 & 1.81±0.36 \\
Shield + SRO (s=50) & 14.37±0.65 & 2.34±0.23 & 7.34±1.27 & \textbf{1.58±0.47} & 0.31±0.20 & 0.53±0.37 & 28.24±4.54 & 1.87±0.63 \\
  \hline
  \end{tabular}%
}
\end{table}

\begin{table}[!htbp]
\centering
\resizebox{\columnwidth}{!}{%
  \begin{tabular}{lcccccccc}
  \hline
  \diagbox{Algo.}{Env.} & \multicolumn{2}{c}{Car-Goal} & \multicolumn{2}{c}{Car-Button} & \multicolumn{2}{c}{Car-Push} & \multicolumn{2}{c}{Car-Circle} \\
  \hline
  & \textbf{R} $\uparrow$ & \textbf{C}(\%) $\downarrow$ & \textbf{R} $\uparrow$ & \textbf{C}(\%) $\downarrow$ & \textbf{R} $\uparrow$ & \textbf{C}(\%) $\downarrow$ & \textbf{R} $\uparrow$ & \textbf{C}(\%) $\downarrow$ \\
  \hline
  RCPO & 15.64±2.27 & 2.23±0.33 & \textbf{5.74±0.34} & 1.89±0.22 & -0.09±0.13 & 0.60±0.09 & \textbf{11.83±0.49} & 1.58±0.57 \\
Shield + SRO (s=5) & 14.43±0.27 & 2.12±0.15 & 5.24±0.90 & 1.86±0.28 & 0.09±0.07 & 0.75±0.24 & 11.55±0.29 & \textbf{1.28±1.02} \\
Shield + SRO (s=10) & 13.65±0.35 & \textbf{1.84±0.14} & 5.52±0.90 & 1.92±0.43 & \textbf{0.23±0.09} & 0.86±0.24 & 11.61±0.67 & 1.44±0.84 \\
Shield + SRO (s=20) & \textbf{15.69±0.47} & 2.51±0.50 & 4.93±0.81 & 1.99±0.20 & 0.04±0.18 & 0.61±0.41 & 11.44±0.20 & 1.88±1.02 \\
Shield + SRO (s=50) & 13.23±1.73 & 2.65±0.86 & 4.88±1.37 & \textbf{1.52±0.17} & 0.00±0.10 & \textbf{0.50±0.20} & 11.08±0.37 & 1.49±0.74 \\
  \hline
  \end{tabular}%
}
\caption{Ablation Study on the Varying Effects of Sampling Numbers $s$ on Safety and Performance with fixed safety bonus $\alpha = 1.0$. Best performances (highest return and lowest cost rate) are highlighted in bold.}
\label{tab:sampling_abl_rcpo}
\end{table}
\FloatBarrier

\hspace{0.5cm}
\section{Adaptive Shielding and Safety-Regularized Optimization with PPO-Lag}\label{apx: adaptive shielding with ppo}
Our method is compatible with a wide range of RL algorithms, as it wraps the policy with a shielding layer and incorporates an augmented term based on $Q_C$ and $V_C$, which are commonly used in safe RL algorithms.
To demonstrate this, we examine its impact when applied to a PPO-Lagrangian-based policy. Here, we study PPO-Lag augmented with SRO and Adaptive Shielding mechanisms. The baseline PPO-Lag method is provided with access to the hidden parameter $\context$.
First, we analyze how SRO affects PPO-Lagrangian method. Then, we show how Shielding mechanism combined with SRO affects PPO-Lagrangian method. All results shown represent the mean reward and cost rate over the last 20 epochs of training across seeds.

\startpara{Safety-Regularized Optimization with PPO-Lag}
Table~\ref{tab:safety_bonus_car_ppolag} demonstrate that SRO generally enhances safety. In Point Robot case, a clear pattern shows that higher safety bonus values $\alpha$ improve safety the most, while lower $\alpha$ have minor effects. Moreover, adding SRO does not degrade reward performance substantially, with cost violations improving by up to 20\% (Point-Button) to as much as 520\% (Point-Circle). Meanwhile, reward degradation occurs only in Point-Circle and Car-Button environments; in many other cases, SRO improves not only safety but also the reward signal. This is primarily because our augmented objective is bounded in $(-1, 0]$, mildly influencing the training objective to compensate for actions leading to zero long-term cost violations and under-explored actions, without causing significant shifts during training due to the bounded values.
\begin{table}[!htbp]
\centering
\resizebox{\columnwidth}{!}{%
  \begin{tabular}{lcccccccc}
  \hline
  \diagbox{Algo.}{Env.} & \multicolumn{2}{c}{Point-Goal} & \multicolumn{2}{c}{Point-Button} & \multicolumn{2}{c}{Point-Push} & \multicolumn{2}{c}{Point-Circle} \\
  \hline
  & \textbf{R} $\uparrow$ & \textbf{C}(\%) $\downarrow$ & \textbf{R} $\uparrow$ & \textbf{C}(\%) $\downarrow$ & \textbf{R} $\uparrow$ & \textbf{C}(\%) $\downarrow$ & \textbf{R} $\uparrow$ & \textbf{C}(\%) $\downarrow$ \\
  \hline
  PPOLag & 18.20±0.78 & 3.35±0.20 & 10.83±1.56 & 2.94±0.22 & 0.24±0.19 & 1.15±0.27 & \textbf{30.90±3.52} & 13.38±6.74 \\
SRO + PPOLag ($\alpha$=0.05) & \textbf{19.41±0.45} & 2.97±0.25 & 10.92±1.65 & 3.16±0.54 & 0.19±0.08 & 1.16±0.10 & 25.20±2.31 & 4.56±1.87 \\
SRO + PPOLag ($\alpha$=0.1) & 17.65±0.26 & 2.52±0.07 & \textbf{11.63±0.80} & 3.04±0.58 & \textbf{0.53±0.27} & 1.32±0.36 & 27.79±1.04 & 2.93±1.49 \\
SRO + PPOLag ($\alpha$=0.5) & 17.12±0.97 & \textbf{2.51±0.50} & 10.49±0.68 & 2.88±0.32 & 0.32±0.03 & 0.85±0.24 & 26.50±3.41 & 3.58±1.07 \\
SRO + PPOLag ($\alpha$=1.0) & 16.98±1.36 & 2.79±0.36 & 9.69±0.55 & \textbf{2.44±0.35} & 0.36±0.20 & \textbf{0.59±0.16} & 26.07±4.88 & \textbf{2.15±1.27} \\
  \hline
  \end{tabular}%
}

\end{table}
\FloatBarrier
\begin{table}[!htbp]
\centering
\resizebox{\columnwidth}{!}{%
  \begin{tabular}{lcccccccc}
  \hline
  \diagbox{Algo.}{Env.} & \multicolumn{2}{c}{Car-Goal} & \multicolumn{2}{c}{Car-Button} & \multicolumn{2}{c}{Car-Push} & \multicolumn{2}{c}{Car-Circle} \\
  \hline
  & \textbf{R} $\uparrow$ & \textbf{C}(\%) $\downarrow$ & \textbf{R} $\uparrow$ & \textbf{C}(\%) $\downarrow$ & \textbf{R} $\uparrow$ & \textbf{C}(\%) $\downarrow$ & \textbf{R} $\uparrow$ & \textbf{C}(\%) $\downarrow$ \\
  \hline
  PPOLag & 15.85±1.01 & 3.36±0.46 & \textbf{8.62±1.26} & 3.58±0.36 & \textbf{0.02±0.09} & 1.14±0.12 & 11.48±0.38 & 2.57±0.23 \\
  SRO + PPOLag ($\alpha$=0.05) & 16.10±1.07 & \textbf{3.08±0.11} & 8.37±1.69 & 3.31±1.13 & 0.02±0.07 & \textbf{0.95±0.20} & 11.24±0.13 & 2.32±0.36 \\
  SRO + PPOLag ($\alpha$=0.1) & \textbf{16.81±1.17} & 3.16±0.18 & 8.53±1.64 & 3.56±0.36 & 0.02±0.04 & 1.07±0.09 & 11.46±0.45 & 2.00±0.20 \\
  SRO + PPOLag ($\alpha$=0.5) & 15.50±1.11 & 3.18±0.34 & 8.24±0.54 & \textbf{3.14±0.43} & -0.02±0.10 & 1.14±0.29 & \textbf{11.52±0.47} & 2.02±0.88 \\
  SRO + PPOLag ($\alpha$=1.0) & 16.41±0.57 & 3.60±0.53 & 8.11±1.17 & 3.78±0.16 & -0.01±0.10 & 0.99±0.23 & 11.42±0.56 & \textbf{1.74±0.60} \\
  \hline
  \end{tabular}%
}
\caption{Ablation Study on the Varying Effects of Safety Bonus $\alpha$ on Safety and Performance. Best performances (highest return and lowest cost rate) are highlighted in bold.}
\label{tab:safety_bonus_car_ppolag}
\end{table}
\FloatBarrier

\startpara{Adaptive Shielding with PPO-Lag}
For sampling numbers, unlike the safety bonus $\alpha$, no consistent pattern emerges due to randomness in the sampling procedure. The same reasoning outlined in Appendix~\ref{apx: sb and ss} applies here. Thus, reducing prediction errors and mitigating the inherent randomness in the sampling process represent key areas for future research to enhance shielding-based approaches.

\begin{table}[!htbp]
\centering
\resizebox{\columnwidth}{!}{%
  \begin{tabular}{lcccccccc}
  \hline
  \diagbox{Algo.}{Env.} & \multicolumn{2}{c}{Point-Goal} & \multicolumn{2}{c}{Point-Button} & \multicolumn{2}{c}{Point-Push} & \multicolumn{2}{c}{Point-Circle} \\
  \hline
  & \textbf{R} $\uparrow$ & \textbf{C}(\%) $\downarrow$ & \textbf{R} $\uparrow$ & \textbf{C}(\%) $\downarrow$ & \textbf{R} $\uparrow$ & \textbf{C}(\%) $\downarrow$ & \textbf{R} $\uparrow$ & \textbf{C}(\%) $\downarrow$ \\
  \hline
  PPOLag & 18.20±0.78 & 3.35±0.20 & 10.83±1.56 & 2.94±0.22 & 0.24±0.19 & 1.15±0.27 & \textbf{30.90±3.52} & 13.38±6.74 \\
Shield + SRO (s=5) & 17.41±1.96 & 2.84±0.36 & 10.27±1.30 & 2.87±0.36 & \textbf{0.38±0.11} & 1.01±0.12 & 25.74±3.26 & 4.25±1.30 \\
Shield + SRO (s=10) & \textbf{18.38±0.84} & 2.86±0.25 & 10.18±0.86 & 2.87±0.23 & 0.20±0.06 & 0.92±0.36 & 27.34±8.09 & 2.42±2.48 \\
Shield + SRO (s=20) & 17.10±2.28 & 2.75±0.32 & 8.84±0.05 & \textbf{2.67±0.27} & 0.15±0.12 & \textbf{0.82±0.10} & 28.83±3.36 & \textbf{2.03±0.87} \\
Shield + SRO (s=50) & 17.44±0.90 & \textbf{2.60±0.26} & \textbf{11.23±3.43} & 3.12±1.07 & 0.37±0.10 & 1.27±0.64 & 27.55±4.32 & 2.43±1.54 \\
  \hline
  \end{tabular}%
}
\end{table}
\FloatBarrier

\begin{table}[!htbp]
\centering
\resizebox{\columnwidth}{!}{%
  \begin{tabular}{lcccccccc}
  \hline
  \diagbox{Algo.}{Env.} & \multicolumn{2}{c}{Car-Goal} & \multicolumn{2}{c}{Car-Button} & \multicolumn{2}{c}{Car-Push} & \multicolumn{2}{c}{Car-Circle} \\
  \hline
  & \textbf{R} $\uparrow$ & \textbf{C}(\%) $\downarrow$ & \textbf{R} $\uparrow$ & \textbf{C}(\%) $\downarrow$ & \textbf{R} $\uparrow$ & \textbf{C}(\%) $\downarrow$ & \textbf{R} $\uparrow$ & \textbf{C}(\%) $\downarrow$ \\
  \hline
  PPOLag & 15.85±1.01 & \textbf{3.36±0.46} & \textbf{8.62±1.26} & 3.58±0.36 & 0.02±0.09 & 1.14±0.12 & \textbf{11.48±0.38} & 2.57±0.23 \\
Shield + SRO (s=5) & 15.56±0.47 & 3.66±0.64 & 7.57±0.39 & 3.28±0.12 & 0.03±0.11 & 1.48±0.46 & 10.38±0.39 & 1.43±0.70 \\
Shield + SRO (s=10) & \textbf{16.87±5.59} & 3.79±0.24 & 7.62±0.91 & 3.35±0.57 & 0.05±0.10 & 1.12±0.30 & 10.15±0.67 & 1.01±0.27 \\
Shield + SRO (s=20) & 15.97±1.36 & 3.43±0.83 & 7.56±1.01 & \textbf{2.64±0.47} & \textbf{0.09±0.14} & 1.22±0.36 & 10.33±0.15 & \textbf{0.93±0.09} \\
Shield + SRO (s=50) & 13.93±2.09 & 3.48±0.50 & 6.25±1.47 & 3.22±0.41 & 0.05±0.10 & \textbf{1.01±0.55} & 11.38±0.54 & 1.49±0.07 \\
  \hline
  \end{tabular}%
}
\caption{Ablation Study on the Varying Effects of Sampling Numbers $s$ on Safety and Performance with fixed safety bonus $\alpha = 1.0$. Best performances (highest return and lowest cost rate) are highlighted in bold.}
\vspace{0.5cm}
\label{tab:sampling_car_ppolag}
\end{table}
\FloatBarrier

\vspace{0.5cm}
\begin{figure}[!t]
    \centering
    \includegraphics[width=1.0\columnwidth]{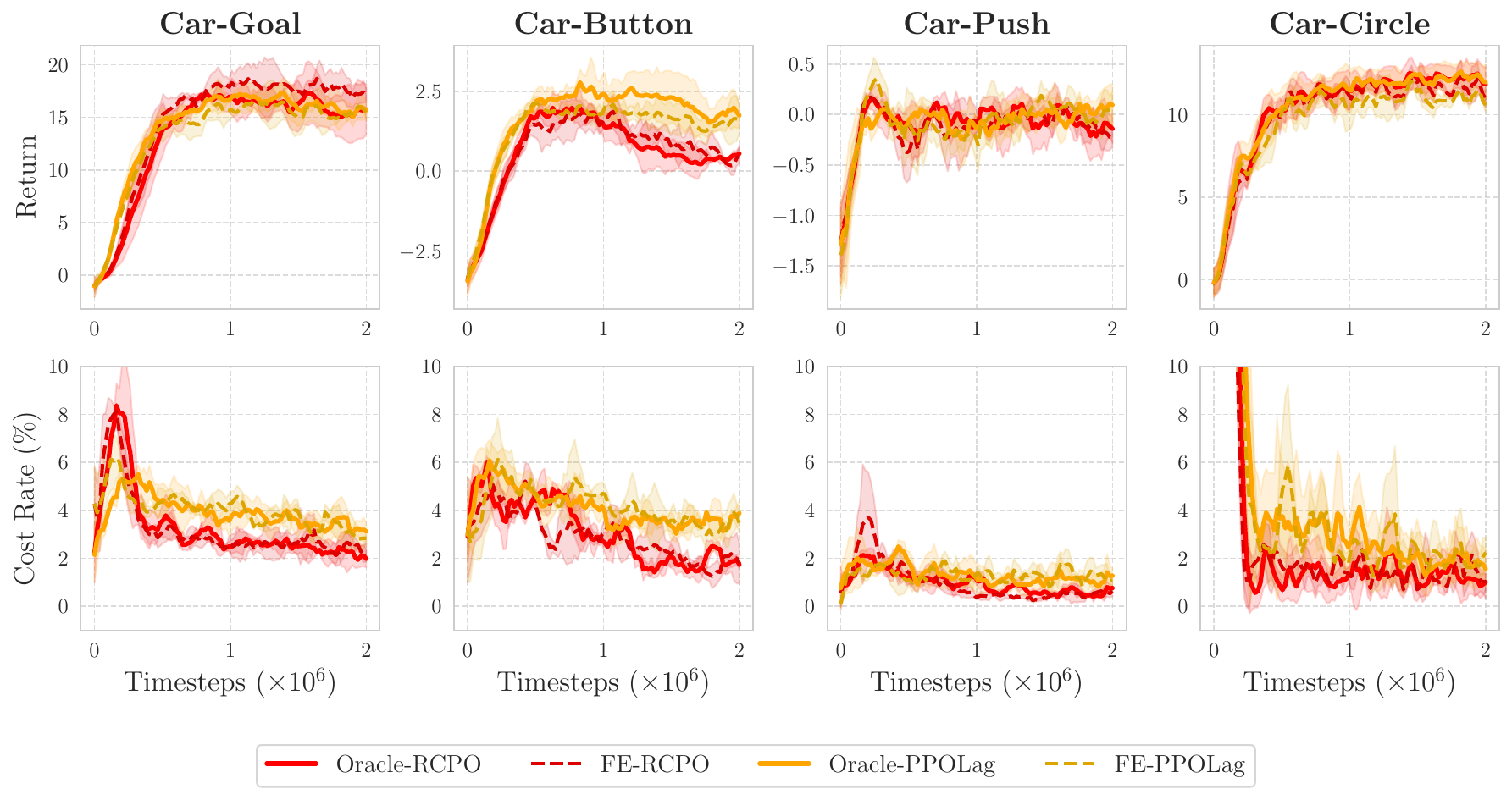}

    \centering
    \includegraphics[width=1.0\columnwidth]{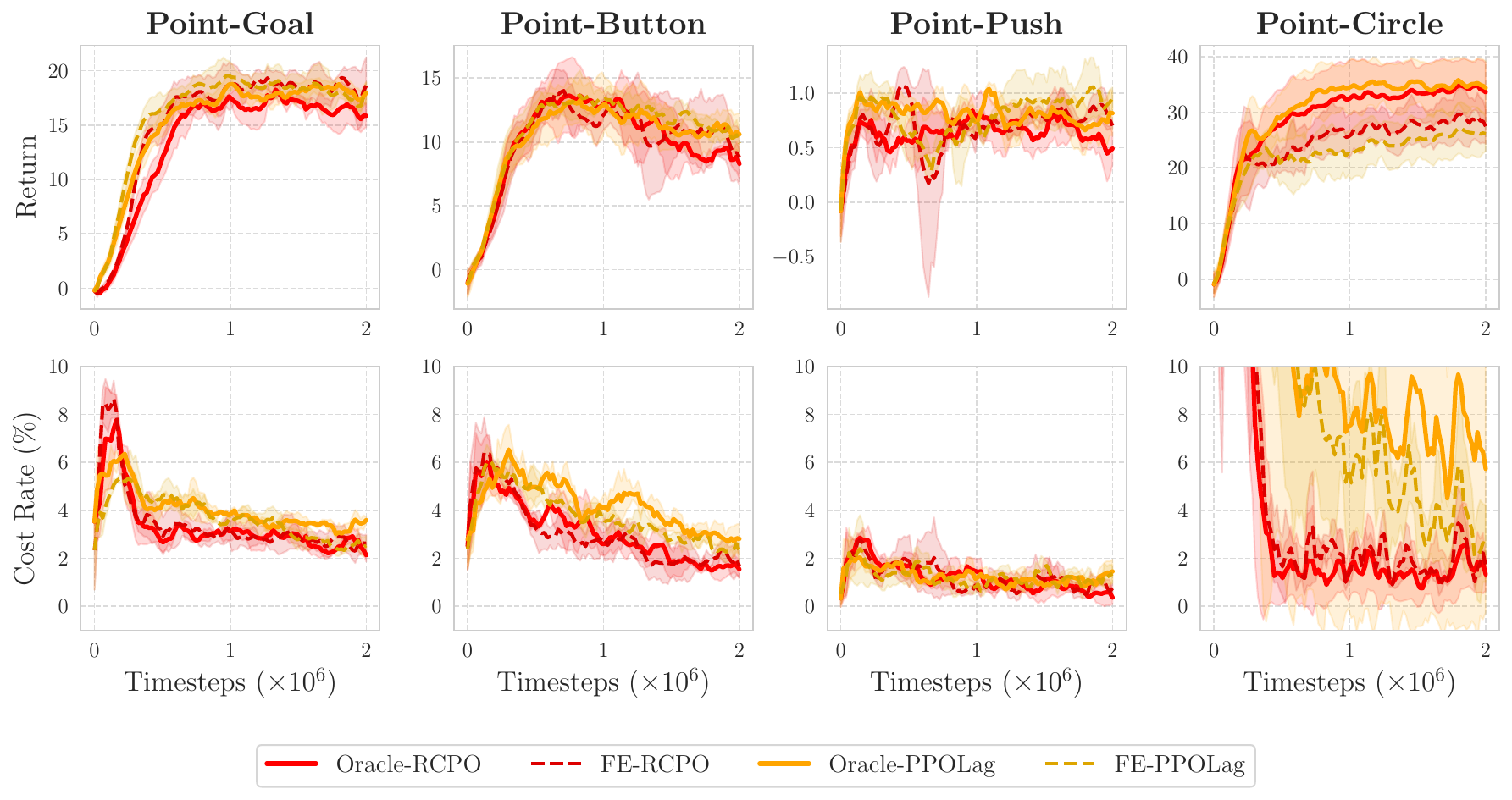}
    \caption{Ablation study on Representation. "Oracle-" refers to a policy directly informed of hidden parameters, while "FE-" denotes the function encoder's representation derived from observations.}
    \vspace{0.5cm}
    \label{fig:rep_point}
\end{figure}
\section{Function Encoders and Its Representation}\label{apx: representation}
For completeness, we present the necessary background on function encoders. For more information, see \citep{ingebrand2024zeroshot_ode, ingebrand2025functionencodersprincipledapproach}.

\startpara{Overview of Function Encoders}
Function encoders, as introduced by \cite{ingebrand2025functionencodersprincipledapproach}, provide a principled framework for representing tasks in a Hilbert space $\mathcal{H}$ through a finite set of neural network-based basis functions $\{g_1, \ldots, g_k\}$, each parameterized by $\theta_j$. A task function $f \in \mathcal{H}$ is approximated as a linear combination of these basis functions:
\[
f(x) = \sum_{j=1}^k \rep_j g_j(x \mid \theta_j),
\]
where $\rep_j$ are coefficients tailored to the specific task. This approach enables efficient representation of complex functions, such as those encountered in reinforcement learning or classification, by learning a versatile basis that spans the function space. By defining appropriate inner products, function encoders can generalize to various function spaces, including probability distributions for classification tasks. \added{
During supervised training, the neural basis functions $\{g_j\}$ are optimized such that $\sum_j b_j g_j(x_i) \approx f(x_i)$ for all training points $(x_i, f(x_i))$. 
As expected from functional analysis, using more basis functions increases expressive power and allows modeling more complex dynamics. Empirically, we demonstrate that a larger number of basis functions yields richer representations showing faster convergence of the dynamics prediction loss. Models with fewer basis functions eventually reach similar final performance but require more training epochs (Appendix~\ref{apx:alternative dynamics predictors}).
}

The training process consists of two phases: offline training of the basis functions and online inference to compute task-specific coefficients. The offline phase optimizes the basis to minimize reconstruction error across a set of source datasets, while the online phase efficiently computes coefficients for new tasks using the learned basis.

\startpara{Training Function Encoders via Least Squares}
Function encoder is trained by using a least-squares optimization approach~\citep{ingebrand2025functionencodersprincipledapproach}. Given a set of task functions $\{f_1, \ldots, f_n\}$, the goal is to learn a set of basis functions $\{g_1, \ldots, g_k\}$ parameterized by $\theta$ and these basis functions represent the task functions with varying coefficients $\rep$. In our implementation, each task function $f_\ell$ is defined by fixing one set of hidden parameters: gravity, mass, damping, density, friction. 

The training procedure iteratively minimizes a loss function comprising two components: a reconstruction loss and a regularization term. For each task function $f_\ell$, we compute coefficients $b^\ell = [\rep_1^\ell,   \ldots, \rep_k^\ell]^T$ that best approximate the target function $f_\ell$ as:
\[
\rep^\ell = \left[ \begin{array}{ccc}
\langle g_1, g_1 \rangle_{\mathcal{H}} & \cdots & \langle g_1, g_k \rangle_{\mathcal{H}} \\
\vdots & \ddots & \vdots \\
\langle g_k, g_1 \rangle_{\mathcal{H}} & \cdots & \langle g_k, g_k \rangle_{\mathcal{H}}
\end{array} \right]^{-1} \left[ \begin{array}{c}
\langle f_{\ell}, g_1 \rangle_{\mathcal{H}} \\
\vdots \\
\langle f_{\ell}, g_k \rangle_{\mathcal{H}}
\end{array} \right],
\]
where $\langle \cdot, \cdot \rangle_{\mathcal{H}}$ denotes the inner product in the Hilbert space, estimated via Monte Carlo integration over collected data points $\{(x_1, f_\ell(x_1)), ((x_2, f_\ell(x_2)), \cdots, ((x_N, f_\ell(x_N))\}$. The reconstructed function is then $\hat{f}_{\ell} = \sum_{j=1}^k \rep_j^\ell g_j$. The reconstruction loss is defined as:
\[
L = \frac{1}{n} \sum_{i=1}^n \| f_{i} - \hat{f}_{i} \|_{\mathcal{H}}^2,
\]
which measures the average squared error between the true and approximated functions. To ensure the basis functions remain well-conditioned, a regularization term is added to the loss function:
\[
L_{\text{reg}} = \sum_{i=1}^k \left( \| g_i \|_{\mathcal{H}}^2 - 1 \right)^2,
\]
which encourages the basis functions to have unit norm. For a learning rate $\alpha$, the parameters $\theta$ are updated via gradient descent: $\theta \leftarrow \theta - \alpha \nabla_\theta (L + L_{\text{reg}})$, until convergence.

\startpara{Empirical Evaluation of the Representation}
We investigate the function encoder’s representation of varying underlying dynamics $T_\context$. We evaluate two representations for handling hidden parameters in our safe RL framework: Oracle representation, where the hidden parameter $\context$ (a scaling factor for environmental dynamics such as density and damping) is directly provided to the policy by concatenating it with the state input, and Function Encoder (FE) representation, which uses a function encoder $\hat{f}_{\text{FE}}$ to infer the underlying dynamics $T_\context$, with coefficients of pretrained basis functions serving as the representation. These representations are tested to assess the function encoder’s ability to adapt to varying dynamics in Safety Gymnasium tasks. Regarding training hyperparameters, we employed three number of basis functions for the function encoder, trained over 1000 epochs on a dataset of 1000 episodes. Batch size was set to 256. Figure \ref{fig:rep_point} show that the function encoder’s representation is often comparable to the oracle representation and, in some cases, outperforms it. The function encoder leverages neural basis functions to represent the space of varying dynamics $\left\{T_\context\right\}_{\context \in \contextset}$. For instance, just as the $\mathbb{R}^2$ plane is spanned by linear combinations of basis vectors $(0,1)$ and $(1,0)$, the dynamics space is captured by neural basis functions, making their coefficients highly informative. This representation often transitions smoothly, as shown in \citep{tyler2024zero_shot}, promoting policy effective adaptation to dynamic changes. Consequently, our function encoder's representation frequently matches or surpasses oracle representation performance.

\begin{figure}[!t]
    \centering
    \includegraphics[width=0.9\columnwidth]{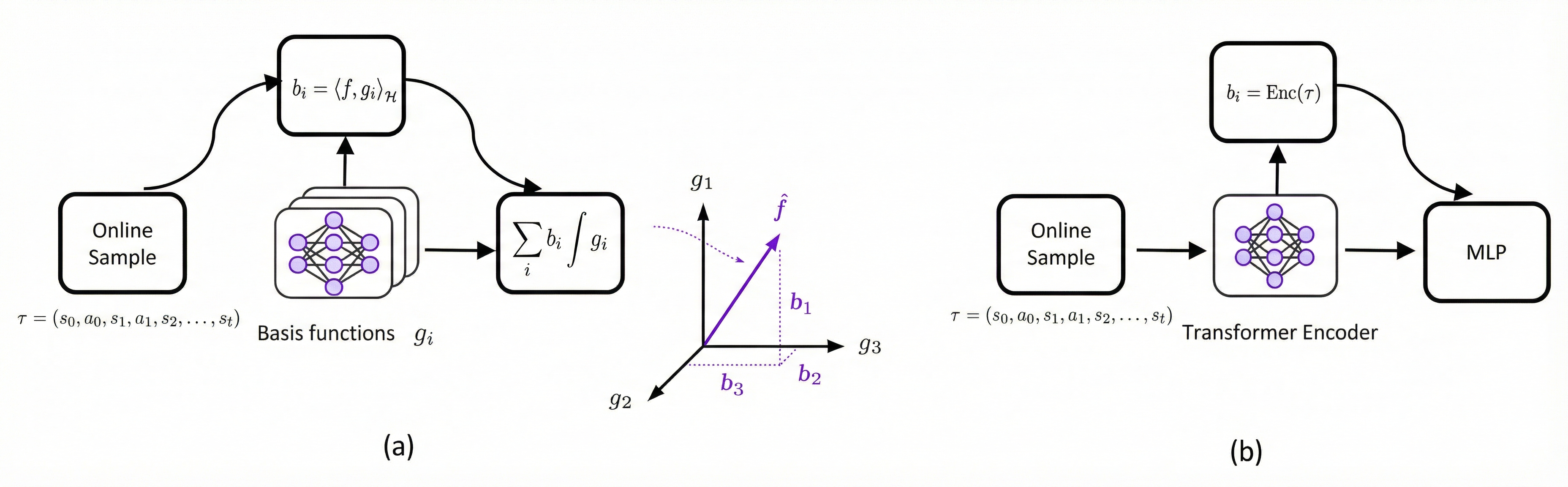}
    \caption{\added{\textbf{(a)} Illustration of how function encoders obtain proxy representations of the underlying hidden parameters using online samples.
\textbf{(b)} A naive approach using transformer encoder to infer hidden parameters from a sequence of online samples. MLP stands for multi-layer perceptron in the Figure.}}
    \label{fig:structure}
\end{figure}
\FloatBarrier

\added{
\section{Shielding with Alternative Dynamics Predictors}\label{apx:alternative dynamics predictors}
In this section, we present ablation studies on different dynamics predictors by replacing the function encoder with alternative prediction models. We consider three predictors: a naive transformer-based dynamics model, a probabilistic ensemble model (PEM), and a multilayer perceptron (MLP). Since PEM and MLP are not designed to infer hidden environment parameters $\phi$ from context alone, we provide them with the true $\phi$ as additional input. Following prior baselines that assume access to the environment parameter~$\phi$, we refer to these models as Oracle-PEM and Oracle-MLP.
\newline{}\\
As another baseline described in Figure \ref{fig:structure}, we use a transformer encoder to process the current episode's trajectory
$$
\tau_n=\left(s_0, a_0, s_1, a_1, \ldots, s_n\right)
$$
and extract a latent representation $b_\phi=\operatorname{Enc}\left(\tau_n\right)$ as a proxy for the hidden parameters $\phi$. For a fair comparison to function encoder, the transformer encoder is given $100$ samples to infer $b_\phi$, which is then concatenated to the state $s_t$ and fed into the predictor to generate $\hat{s}_{t} = (s_t, b_\phi)$. All dynamics predictor are designed to hold a comparable parameter numbers (270k-280k). All dynamics predictors are trained on $1,000$ in-distribution episodes and evaluated on $200$ out-of-distribution test episodes, which are not used during training. 
\newline{}\\
We first report next state prediction performance on evaluation OOD dataset across all dynamics predictors. We then compare the naive transformer combined with Shield and Shield + SRO against the function encoder combined with the same shielding mechanisms. Note that SRO alone does not use dynamics prediction at deployment.
\newline{}\\
Figure \ref{fig:dyna_all_compare} shows that the function encoder's next-state prediction accuracy closely matches that of Oracle-PEM and Oracle-MLP. This observation is consistent with prior findings \cite{tyler2024zero_shot, ingebrand2024zeroshot_ode}, which also report that function-encoder-based models can approximate the performance of oracle-informed predictors.
\newline{}
Instead of inferring the hidden parameters $\phi$ through neural basis functions, the Transformer encodes the observation sequence $\tau=( s_0, a_0, s_1, \ldots, s_n )$ and projects the final embedding into a low-dimensional parameter estimate $b_\phi$. Since the true hidden parameters $\phi$'s dimension in safe-navigation domain is 4 (damping, mass, inertia, friction), we evaluate projection dimensions of $3,6,9$, and $12$.
\newline{}
Figure \ref{fig:rep_hidden_param} shows an interesting pattern. For the function encoder, increasing the number of basis functions shows early-epoch convergence because the representation has higher capacity, but all configurations eventually converge to similar accuracy, which aligns with the supervised nature of the objective and dataset limitations. In contrast, the Transformer-based encoder does not exhibit a consistent relationship between projection dimension and prediction quality, and its overall performance is less stable.
\newline{}\\
Finally, Figure \ref{fig:rep_compare_with_tf} shows that the function encoder consistently yields superior performance, achieving higher rewards and fewer constraint violations. This improvement originates from the combination of $(i)$ a more stable and expressive learned representation and $(ii)$ higher next-state prediction accuracy, both of which enhance shielding framework.
}
\begin{figure}[!t]
    \centering
    \includegraphics[width=1.0\columnwidth]{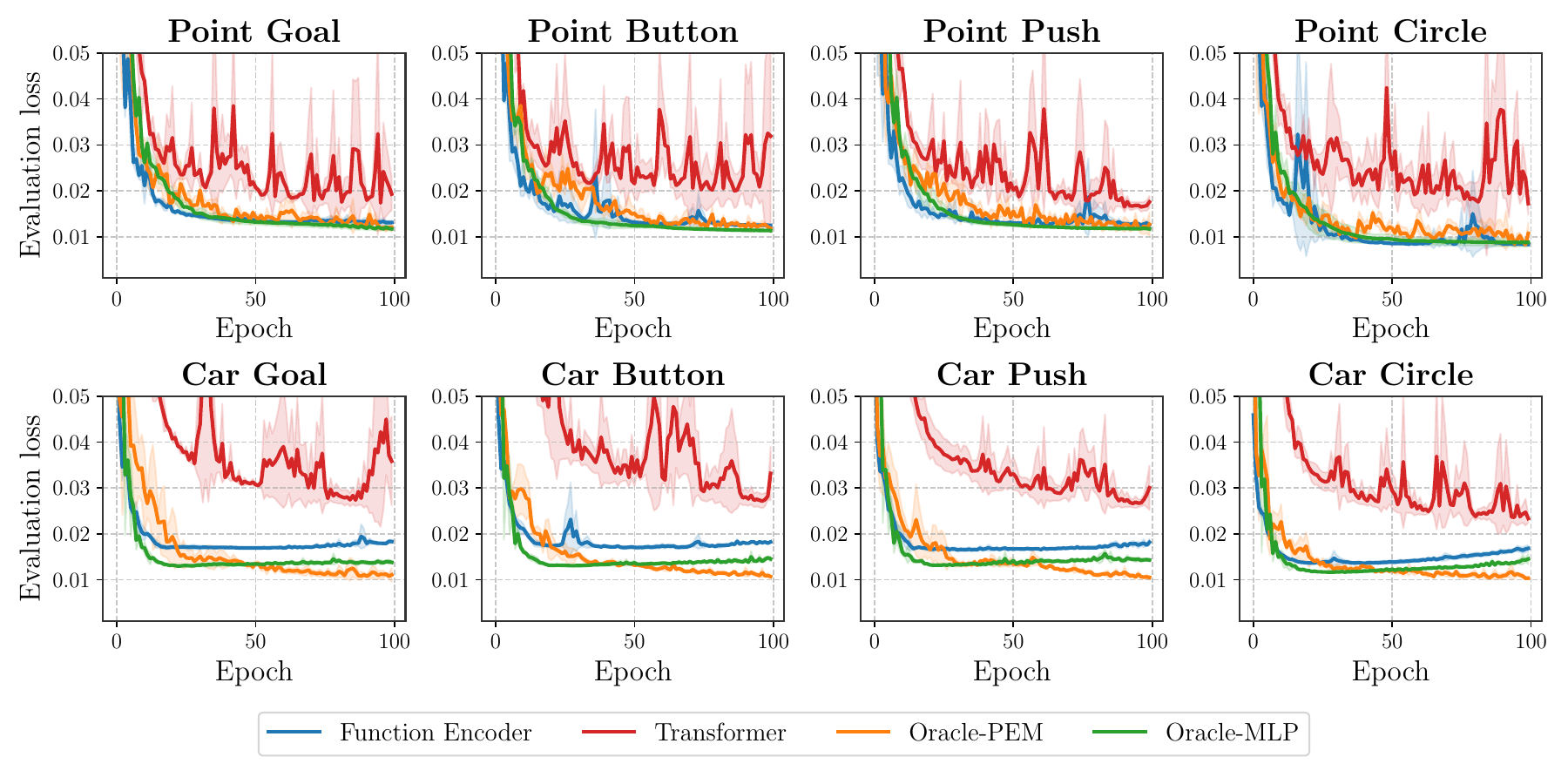}
    \caption{\added{Ablation study evaluating the performance of various dynamics predictors in forecasting the next state. The $y$-axis denotes the average per-sample $\ell_1$-norm error between the true and predicted next states on the test dataset.}}
    \label{fig:dyna_all_compare}
\end{figure}

\begin{figure}[!t]
    \centering
    \includegraphics[width=1.0\columnwidth]{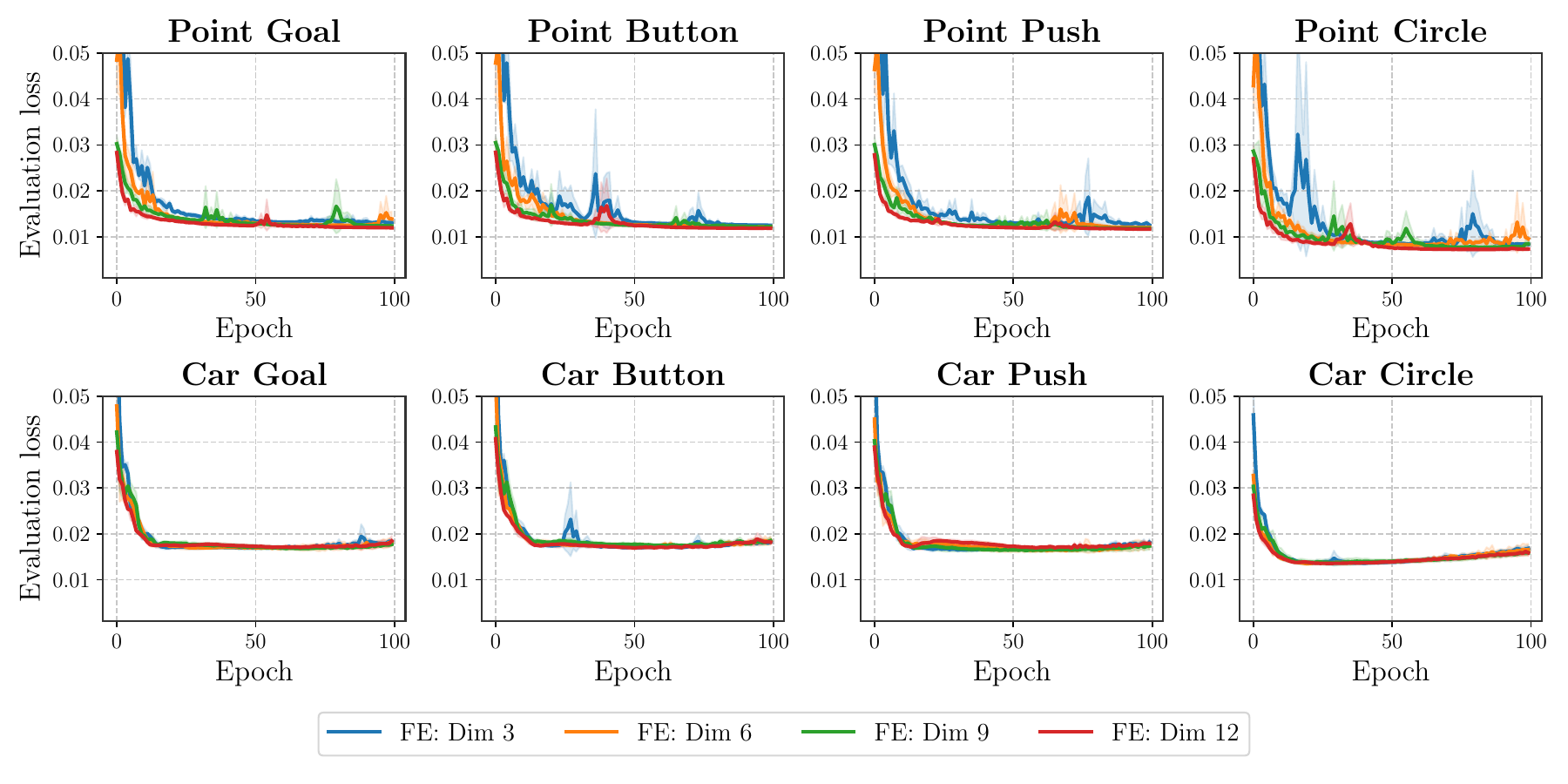}
    
    \centering
    \includegraphics[width=1.0\columnwidth]{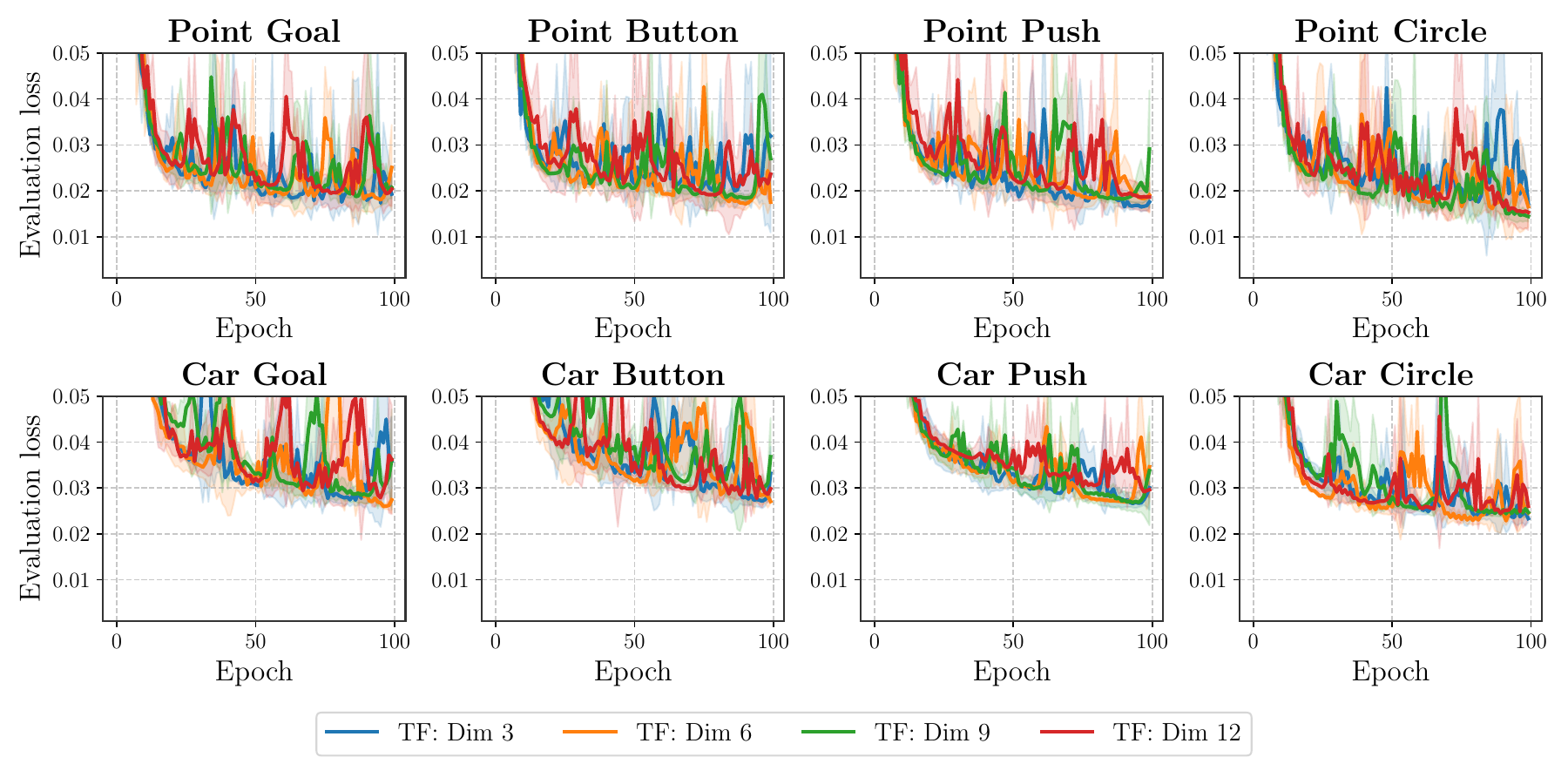}
    \caption{\added{We evaluate how the dimension of the inferred representation $b_\phi$, which is a proxy to hidden parameters $\phi$. We vary the dimension across $\{3,6,9,12\}$, motivated by the ground-truth hidden parameter dimension of $4$ (capturing variations in damping, mass, inertia, and friction). The plot reports the average per-sample $\ell_1$ prediction error between the true and predicted next states on the test set.}}
    \vspace{0.5cm}
    \label{fig:rep_hidden_param}
\end{figure}

\begin{figure}[!t]
\centering

\includegraphics[width=1.0\columnwidth]{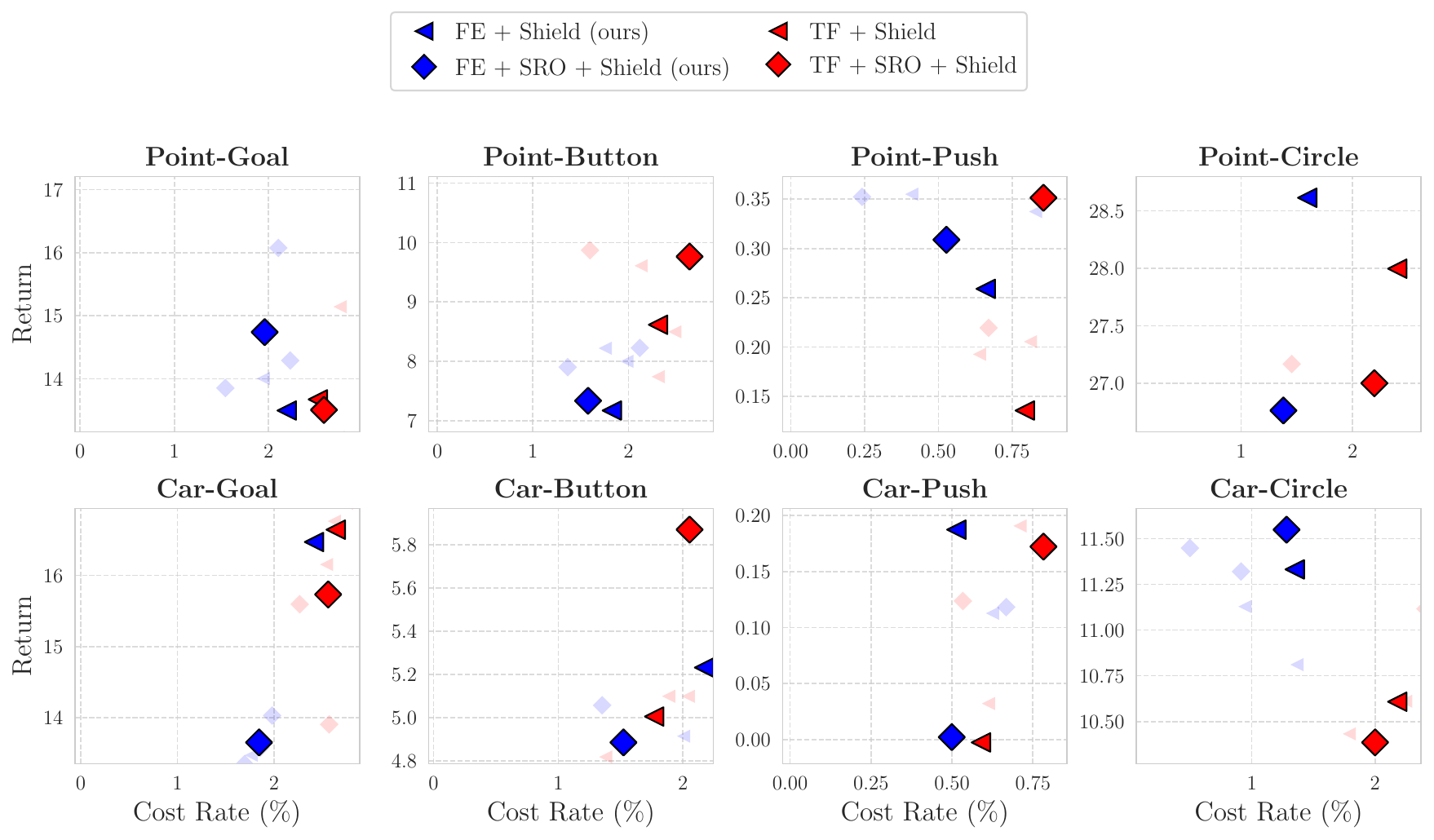}
\raisebox{1cm}{(a) Training performance} 
\includegraphics[width=1.0\columnwidth]{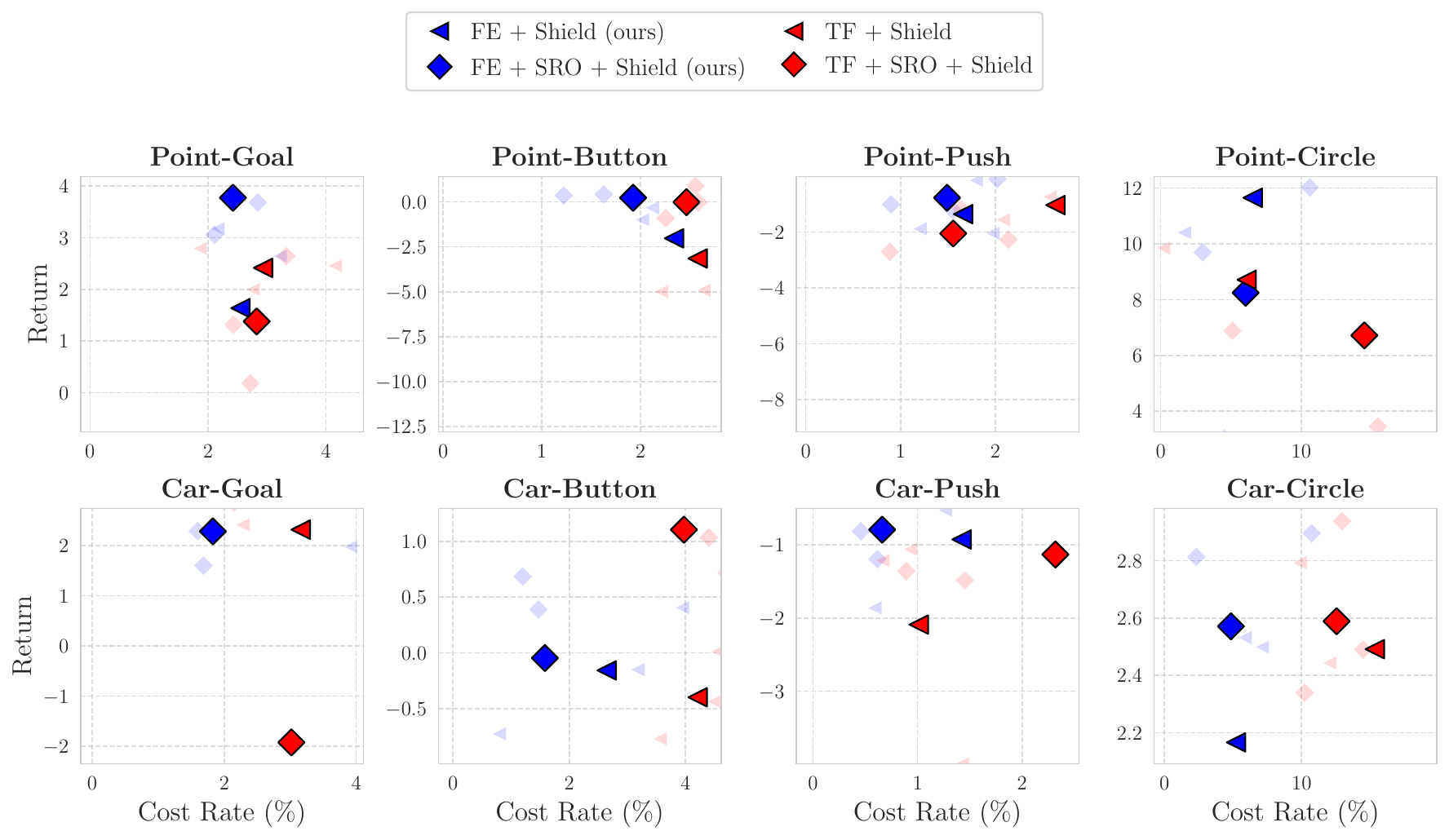}
\vspace{-3cm}
\raisebox{3.9cm}{(b) OOD generalization performance}
\caption{\added{Comparison of our function encoder (FE) representation to a naive Transformer-based (TF) representation when both are used with the proposed shielding framework (all other hyperparameters and conditions remain the same). 
\textbf{(a)} In-distribution training (return vs. cost rate). 
\textbf{(b)} Out-of-distribution (OOD) evaluation (return vs. cost rate). 
In both settings, our method attains higher returns and fewer constraint violations. This improvement arises from the combination of a more informative latent representation and higher next-state prediction accuracy, both of which strengthen the effectiveness of the shielding mechanism.}}
\label{fig:rep_compare_with_tf}
\vspace{0.5cm}
\end{figure}

\added{
\section{Additional Experiments: Safe Velocity Control in HalfCheetah}\label{apx: half cheetah}
\newpage
\begin{wrapfigure}{r}{0.4\columnwidth}
    \centering
    \includegraphics[width=0.3\columnwidth]{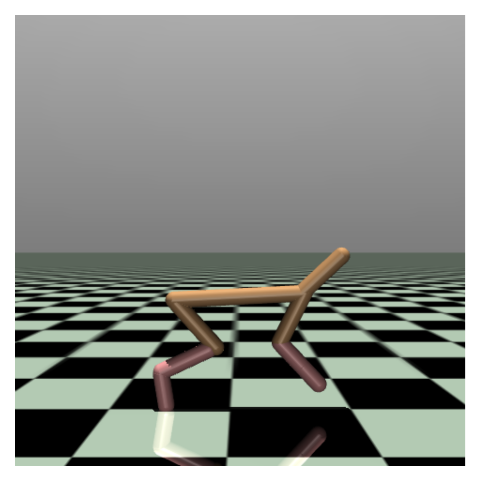}
    \caption{HalfCheetah-Velocity environment.}
    \label{fig:halfcheetah_env}
\end{wrapfigure}
In this section, we evaluate whether our mechanism generalizes to different robot morphologies and task type by conducting shielding experiments in the HalfCheetah-Velocity task. 
We follow the safetyconstrained velocity environment from Safety Gymnasium~\citep{ji2023safety}, but strengthen the safety requirement by tightening the velocity limit from 3.0 to 2.0. As in our earlier settings, the hidden environment parameters vary each episode.
To induce richer dynamics variability, we modify the HalfCheetah default parameters, including friction, body segment lengths, and gear ratios with total $14$ different hidden parameters:
$$\{friction, torso\_length, bthigh\_length, \cdots , foot\_gear\},$$ and resample them at the beginning of each episode. During training, each parameter is scaled uniformly within $[0.7,1.3]$. For out-of-distribution (OOD) evaluation, we use the disjoint ranges $[0.4,0.7] \cup [1.3, 1.6]$.
\newline{}\\
We use 5 neural basis functions in the function encoder, resulting in a 5-dimensional representation that serves as a proxy $b_\phi$ for the hidden parameters $\phi$. As a baseline, RCPO is run with oracle access to the true hidden parameters $\phi$, whereas our method uses the same RCPO implementation without oracle information and instead augment the state with the learned basis coefficients $b_\phi$ as its parameter estimate. All models share the same RL hyperparameters. For hyperparameters for shielding, we use 10 samples of actions, and a safety bonus $\alpha = 1$.
\newline{}\\
Across both training and OOD settings, Figure \ref{fig:half_cheetah} shows consistent improvements in safety and overall performance. Our method achieves a favorable reward-cost Pareto frontier compared to the oracle informed baseline, demonstrating that the learned representation and shielding mechanism transfer effectively to more complex robot dynamics.
\begin{figure}[!t]
    \centering
    \includegraphics[width=0.6\columnwidth]{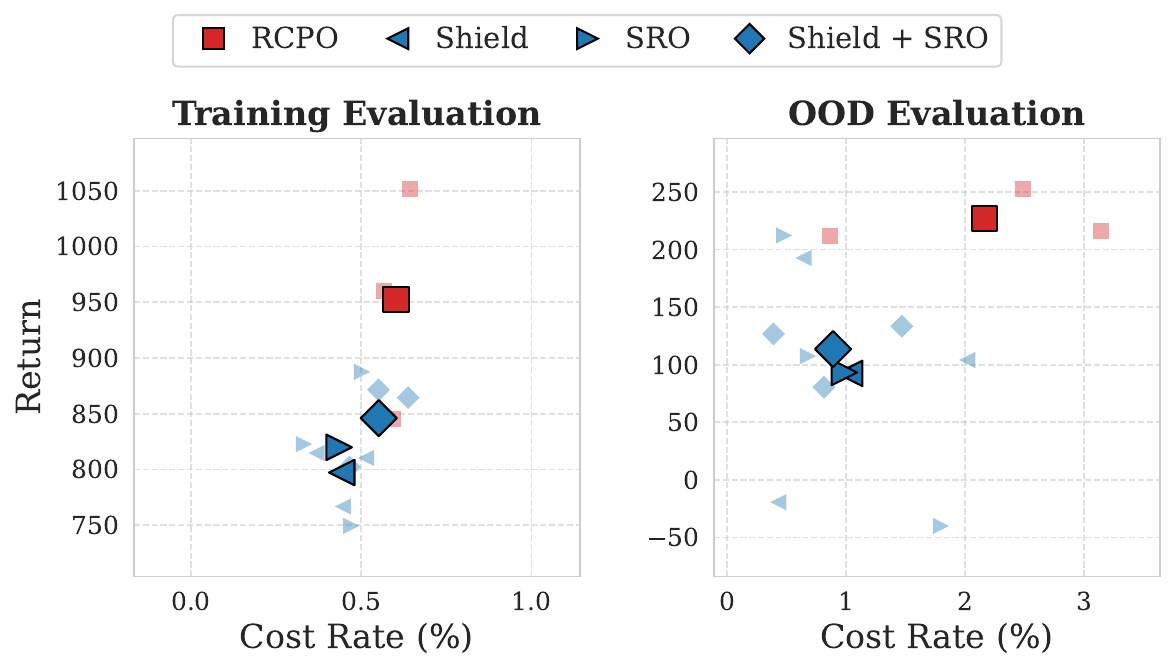}
    \caption{\added{(Left) Return vs. cost rate during the last 20 training epochs, showing the 
reward–safety tradeoff achieved by each method. 
(Right) Out-of-distribution (OOD) evaluation performance under shifted hidden parameters.}}
    \label{fig:half_cheetah}
\end{figure}
}

\section{Why Not $Q_C$, But $Q_{\text{safe}}$?}\label{apx:design choice}
To effectively guide the policy toward safe behavior, we propose a safety-regularized optimization enhanced with $Q_{\text{safe}}$. A natural alternative is to augment the reward value function $Q_R$ with the cost value function $Q_C$, which estimates the expected cost of violating safety constraints, forming $Q_{\text{aug}} = Q_R - \alpha Q_C$. However, this formulation can be transformed into Lagrangian-based safe RL methods, optimizing policies with $Q_R - \lambda Q_C$, where $\lambda$ is a Lagrangian multiplier dynamically adjusted during training. In particular, the Lagrangian multiplier $\lambda$ is updated using a learning rate $lr$. A higher learning rate accelerates the increase of $\lambda$, assigning stronger penalties on the policy for cost violations. $\lambda$ is adjusted by $Q_C \times lr$; larger $Q_C$ or learning rate values lead to faster $\lambda$ growth, which increases the penalty term in the optimization objective ($Q_R - \lambda Q_C$).

However, Primal-Dual methods such as PPOLag or RCPO are sensitive to the choice of $lr$, often leading to unstable optimization or suboptimal safety-performance trade-offs as shown in Figures \ref{fig:lr_rcpo} and \ref{fig:lr_ppolag}. This is because the value of $Q_C$ is highly environment-dependent, varying with the magnitude of costs and the dynamics induced by hidden parameters. In contrast, our safety-regularized objective $Q_{\text{safe}}$ incorporates a normalized term, constrained to $(-1, 0]$. This normalization simplifies controlling the safety bonus by ensuring it remains bounded. 
\begin{figure}[!t]
    \centering
    \includegraphics[width=1.0\columnwidth]{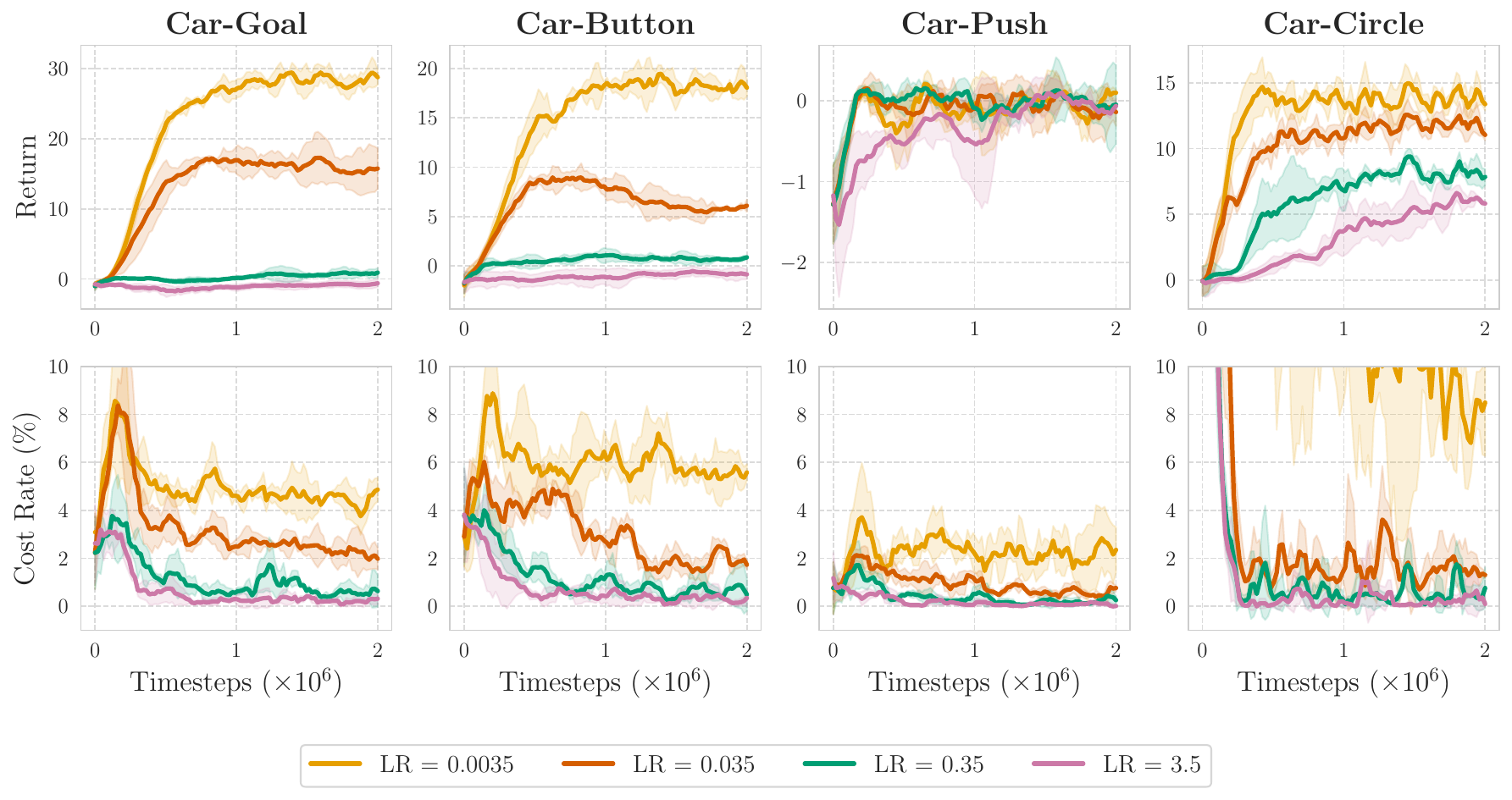}
    \includegraphics[width=1.0\columnwidth]{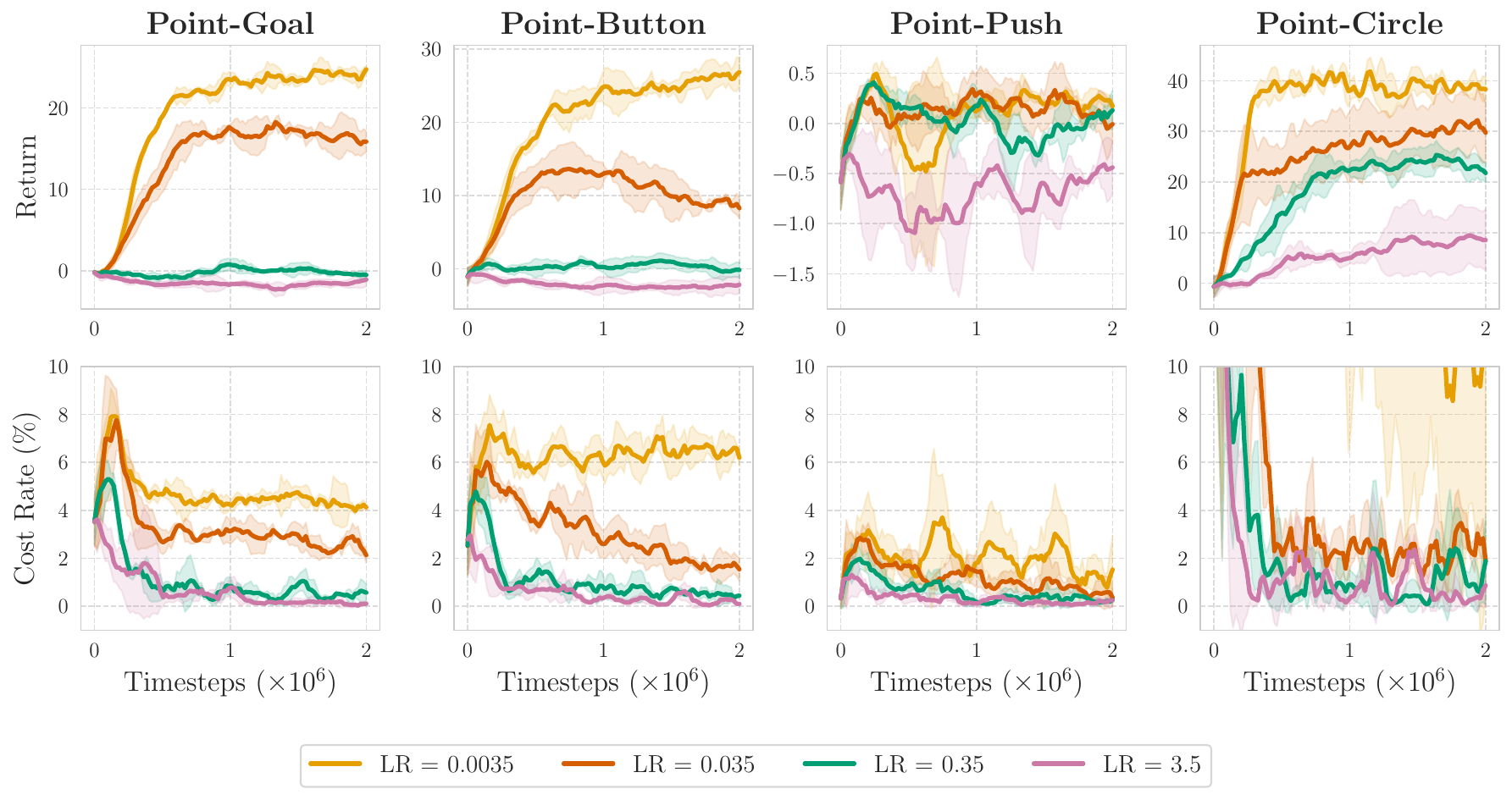}
\caption{Training curves for RCPO algorithm under varying Lagrangian learning rates. The plots illustrate significant performance variations depending on the learning rate. For our main comparisons, we selected a learning rate of 0.035, which achieves the best trade-off between reward maximization and constraint satisfaction.}
    \label{fig:lr_rcpo}
\end{figure}

\begin{figure}[!t]
    \centering
    \includegraphics[width=1.0\columnwidth]{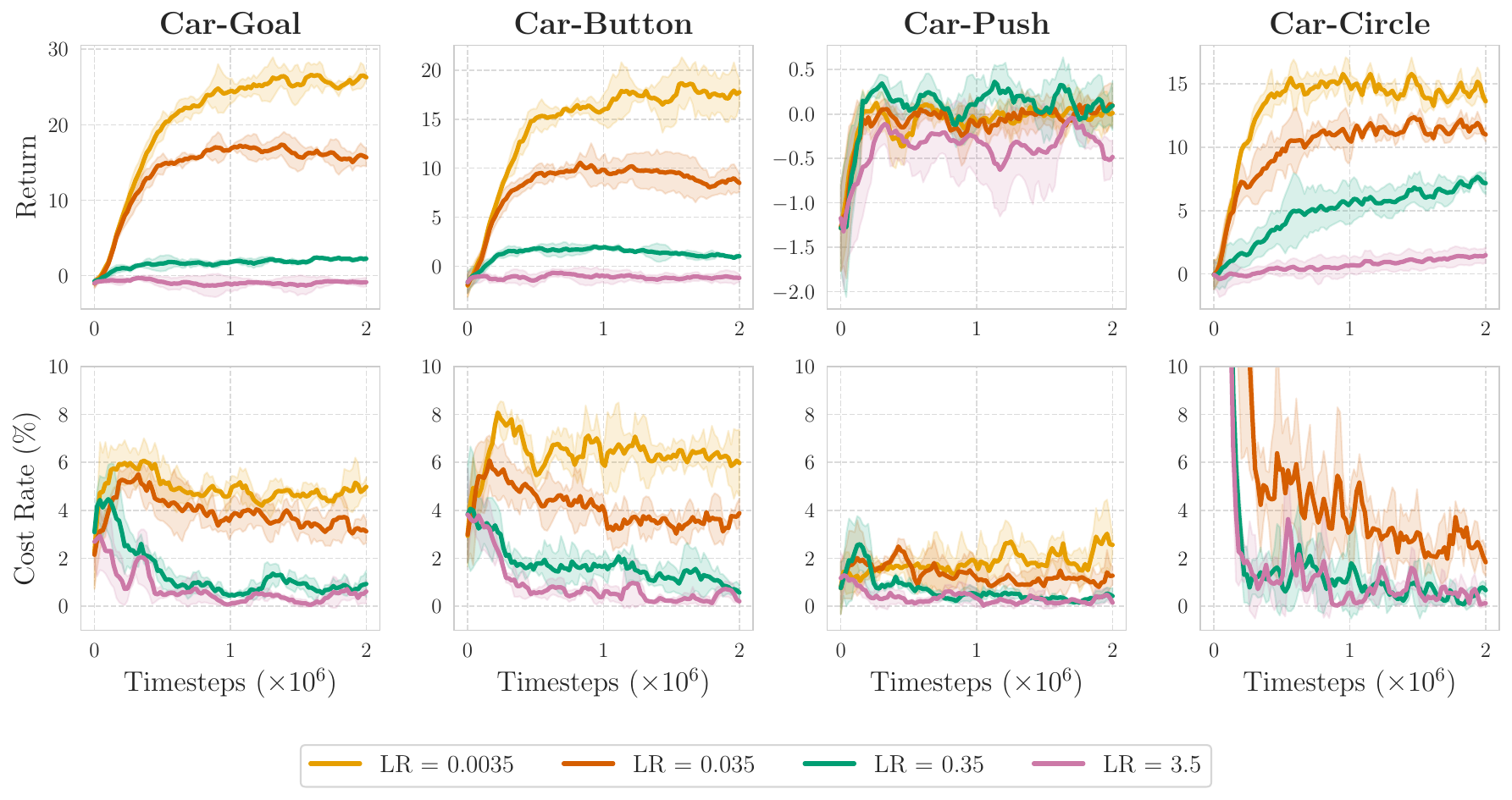}
    \centering
    \includegraphics[width=1.0\columnwidth]{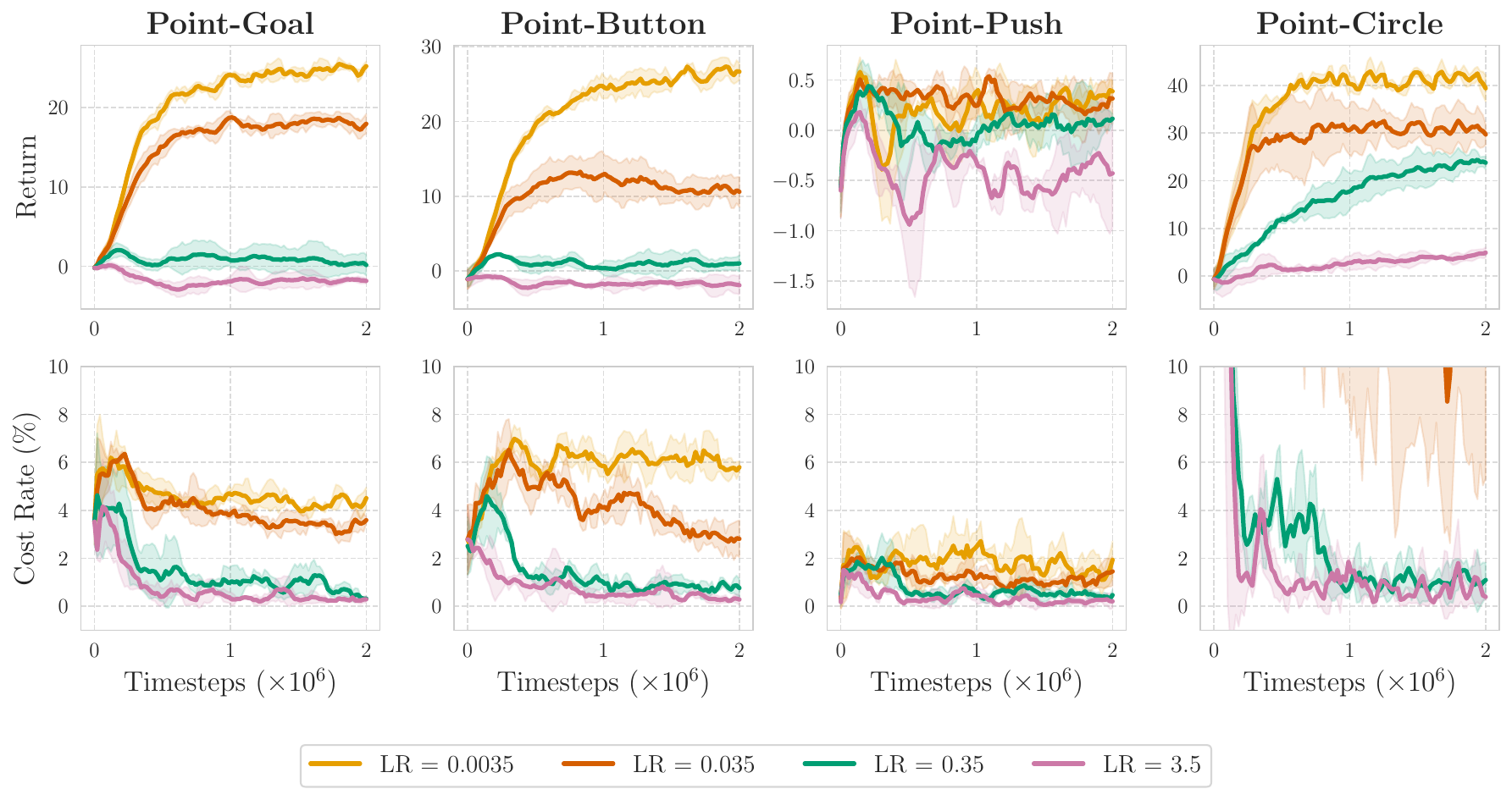}
\caption{Training curves for PPOLag algorithm under varying Lagrangian learning rates. The plots illustrate significant performance variations depending on the learning rate. For our main comparisons, we selected a learning rate of 0.035, which achieves the best trade-off between reward maximization and constraint satisfaction.}
\label{fig:lr_ppolag}
\end{figure}

\section{Fix Parameters vs. Varying Hidden Parameters}\label{apx: hidden_vs_fixed}
\begin{figure}[!t]
\includegraphics[width=1\columnwidth]{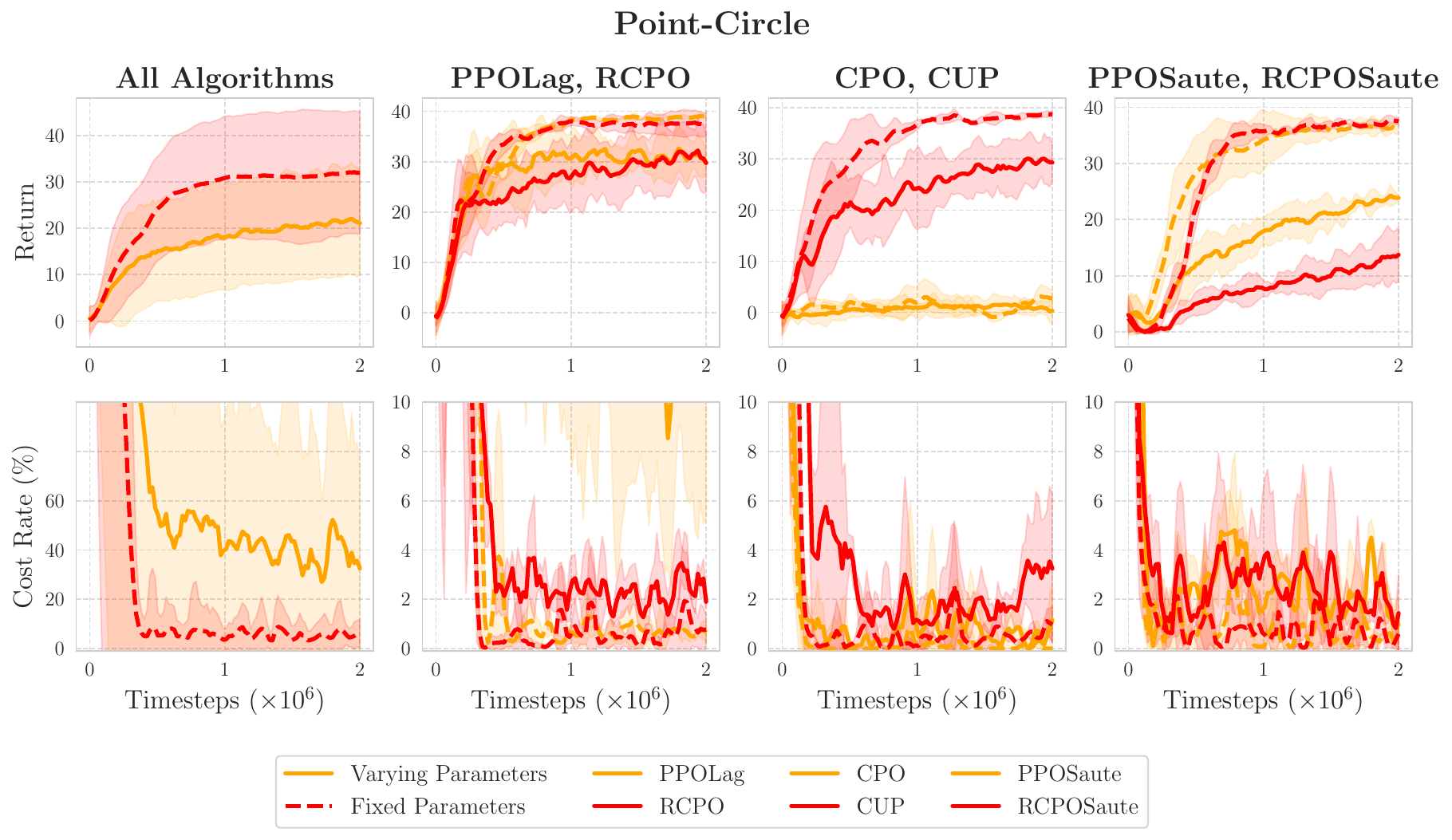}

\includegraphics[width=1\columnwidth]{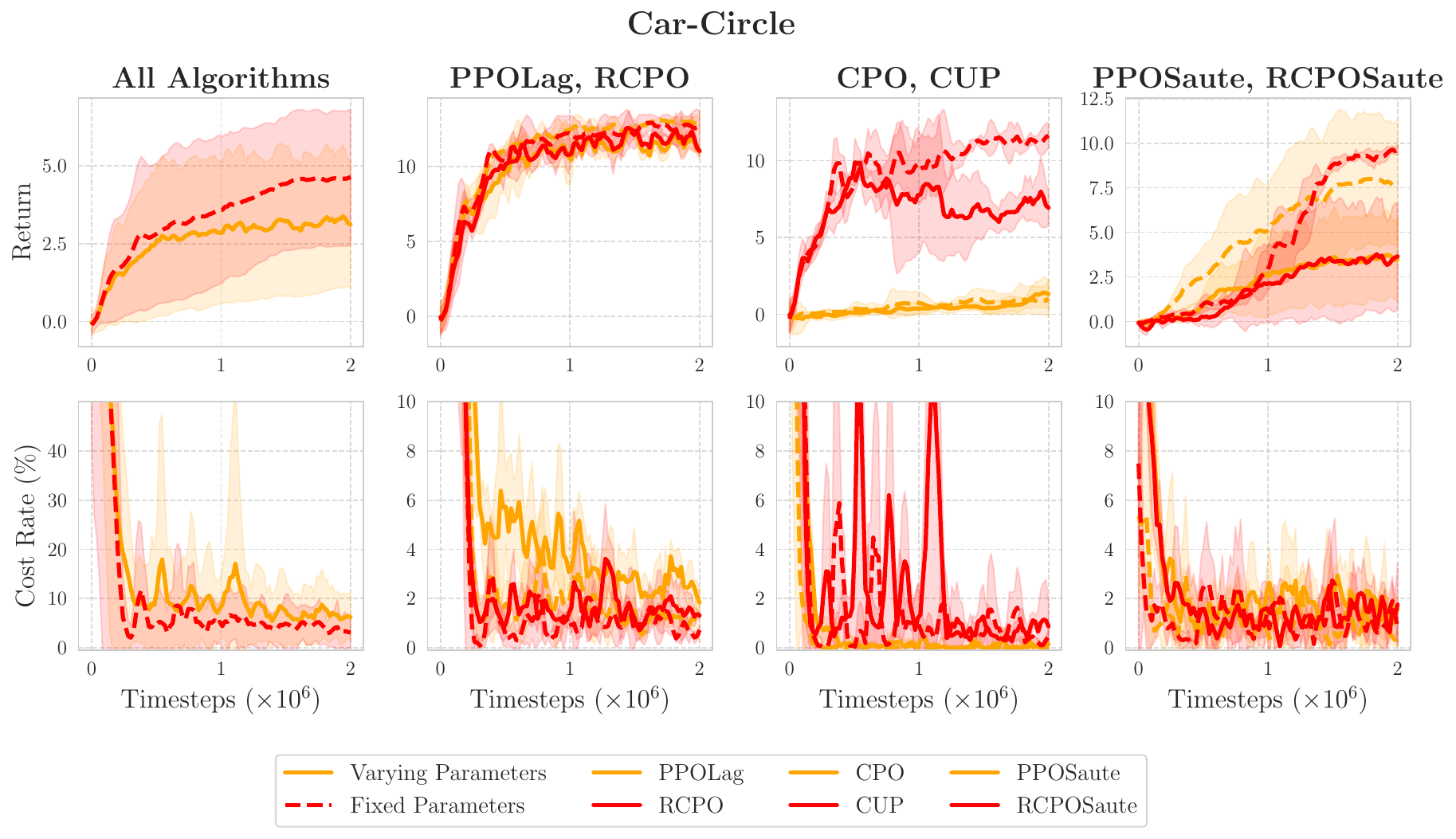}
\caption{Varying hidden parameters pose significant challenges for safe RL algorithms, even when hidden parameter values are provided as input. The first column plot aggregates results across six algorithms: vanilla Lagrangian methods (RCPO and PPO-Lag), trust region/projection methods (CPO and CUP), and safety augmentation techniques (PPO-Saute and TRPO-Saute). Results with varying hidden parameters are shown as solid lines, while those with fixed parameters are depicted as dotted lines. Columns 1, 2, and 3 present comparative results across these algorithm groups.}
\vspace{0.5cm}
\label{fig:hidden-params-challenge}
\end{figure}

We evaluate the algorithms under two distinct experimental settings to test performance difference when environment dynamics shifts:
\begin{tightitems}
    \item \textbf{Fixed Parameters}: In this setting, each algorithm is trained and evaluated in an environment with a single, constant set of parameters $\context$ (gravity, damping, density, mass, friction) for the entire duration of training.
    \item \textbf{Varying Hidden Parameters}: In contrast, for this setting, the underlying physical parameters $\context$ of the environment such as gravity, damping, mass, inertia, and friction are randomized at the start of each new episode. To demonstrate the challenge of adapting to varying hidden parameters, we explicitly inform the algorithms of these changes via their input. For example, if gravity is halved from 9.8 to 4.9, a factor of 0.5 is provided as input to the policy.
\end{tightitems}
As demonstrated in Figure~\ref{fig:hidden-params-challenge}, we observe a noticeable degradation in the performance of all algorithms under the varying parameter setting, despite being explicitly informed of the magnitude of the changes. More precisely, for Point environment, the total aggregated return across training differs significantly: 25.87 for fixed parameters versus 16.49 for varying parameters. Similarly, for the cost, the values are 18.0 for fixed parameters and 38.41 for varying parameters, reflecting more than double the total cost violations during training.

\section{Cost Functions}\label{apx: cost}
In this section, we present two cost functions used in our experiments. 
Each cost function conforms to the form:
\[
C(s_t, a_t, s_{t+1}) = \mathbb{I}\left\{ \nu(e(s_{t+1}), E_{t+1}) \leq 0 \right\},
\]
as defined in Section \ref{subsec: adaptive shielding}, where:
\begin{itemize}
    \item \( e: S \to \mathbb{R}^{n_1}\) extracts agent-centered safety features from the next state \( s_{t+1} \),
    \item \( E_{t+1} \in \mathbb{R}^{n_2} \) captures environment features (e.g., obstacle positions, safe region boundaries),
    \item \( \nu: \mathbb{R}^{n_1} \times \mathbb{R}^{n_2} \to \mathbb{R} \) is a Lipschitz continuous function, with \( \nu > 0 \) indicating safety and \( \nu \leq 0 \) indicating a violation.
\end{itemize}

\begin{table}[h]
    \centering
    \resizebox{\columnwidth}{!}{%
        \renewcommand{\arraystretch}{1.3}
        \begin{tabular}{lccc}
        \toprule
        \textbf{Task} & $e(s)$ & $E_t$ & $\nu\!\bigl(e(s),E\bigr)$ \\
        \midrule
        Collision avoidance &
        $\operatorname{pos}(s)\!\in\!\mathbb{R}^3$ &
        $\{X_i\}_{i=1}^M\!\subset\!\mathbb{R}^3$ &
        $\displaystyle\min_i\|e(s)-X_i\|_2-d_{\text{safe}}$ \\
        Safety-region compliance &
        $\operatorname{pos}(s)\!\in\!\mathbb{R}^2$ &
        $\mathcal{S}_{\mathrm{safe}}\!\subset\!\mathbb{R}^2$ &
        $\displaystyle\operatorname{dist}
           \!\bigl(e(s),\mathbb{R}^2\!\setminus\!\mathcal{S}_{\mathrm{safe}}\bigr)-\varepsilon$ \\
        \bottomrule
        \end{tabular}%
    }
    \caption{Examples of function $\nu$ for different safety tasks.}
    \label{tab:margin-examples}
\end{table}

\startpara{Collision Avoidance} Given the robot's position $\operatorname{pos}(s) \in \mathbb{R}^3$ and the set of obstacle positions $\left\{X_i\right\}_{i=1}^M \subset \mathbb{R}^3$ encoded in the state $s$, we mark a transition unsafe whenever the robot comes closer than a safety margin $d>0$ to any obstacle:
$$
C_d\left(s, a, s^{\prime}\right)=\mathbb{I}\left[\min _i\left\|\operatorname{pos}(s)-X_i\right\|<d\right]
$$
Thus $C_d=1$ whenever the robot violates the distance constraint, encouraging policies that keep a safe distance to obstacles. 

\paragraph{Safety Region Compliance}
To ensure the robot remains within a designated safety region, we evaluate its position in the next state, $\operatorname{pos}(s') = (x, y) \in \mathbb{R}^2$, against a predefined safe region $\operatorname{safe\_region} \subseteq \mathbb{R}^2$. A penalty is incurred if the position lies outside this region:
\[
C_d(s, a, s') = \mathbb{I}[\operatorname{pos}(s') \notin \operatorname{safe\_region}].
\]
This cost function assigns a value of 1 when the robot deviates from the safety region, indicating a safety violation.

\section{Average Cost Minimization and Cost Value Function}\label{apx: avg cost}
Our problem formulation targets minimizing the average cost per time step, distinct from the cumulative discounted cost over an infinite horizon typically addressed by Lagrangian-based methods like TRPO-Lag and PPO-Lag. The connection between cumulative discounted cost and average cost is well-established~\citep{puterman2014mdp}:
$$
\lim_{\gamma \rightarrow 1^{-}}(1 - \gamma)V_\cost^\pi(s_0, \context_0) = \xi^\pi(s_0, \context_0),
$$
where $V_\cost^\pi(s_0, \context_0)$ denotes the value function for the cost under policy $\pi$ starting from state $s$, parameter $\context$, and $\xi^\pi(s, \context) = 
\lim _{H \rightarrow \infty} \frac{1}{H} \mathbb{E}_{\pi^{*}, T_\context}\left[\sum_{t=0}^{H-1} C_d(s_t, a_t, s_{t+1}) \mid s_0 = s, \context \right]$ represents the expected average cost for parameter $\context$. 
\section{Experimental Details}\label{apx: exp}
\begin{figure}[!t]
    \centering
    \includegraphics[width=1.0\columnwidth]{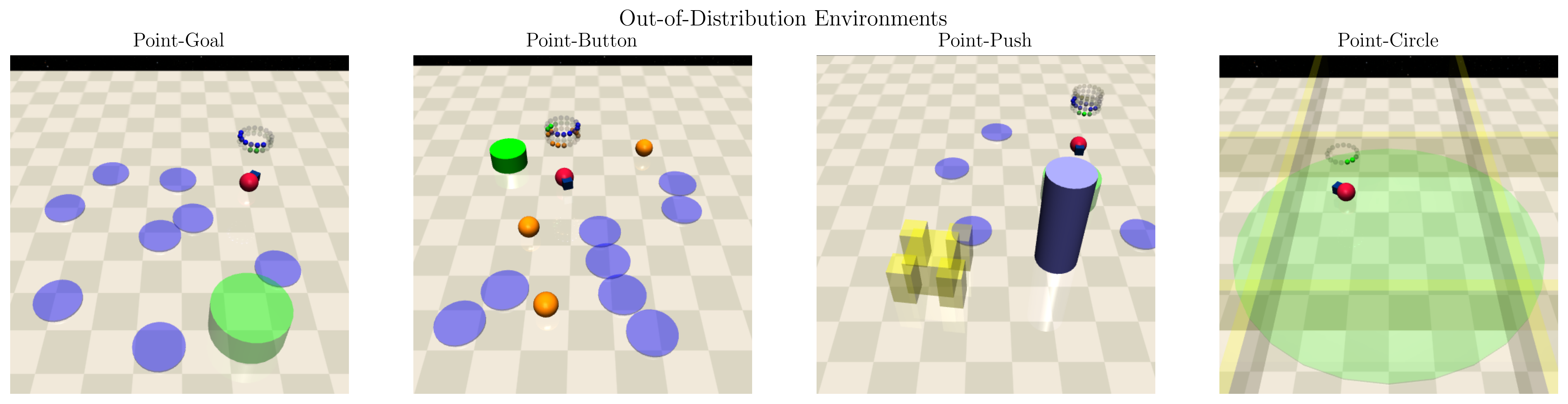}
    \caption{Four out-of-distribution environments for evaluation.}
    \label{fig:envs}
    \vspace{0.5cm}
\end{figure}

For out-of-distribution (OOD) evaluation, we modify Safety Gymnasium task environments: Goal, Button, and Push, by adding two additional hazard locations to increase complexity. For Circle, we keep the same layout since the wall already blokcs four sides.

To introduce varying hidden parameters, each episode independently samples gravity, damping, mass, inertia, friction multipliers by randomly selecting one of two intervals, $[0.15, 0.3]$ or $[1.7, 2.5]$, with equal probability and uniformly sampling a value from the chosen interval, ensuring diverse environmental conditions. For Circle task, all settings remain the same except for damping, which is sampled from $[1.7, 2.5]$. This adjustment addresses instability in MuJoCo simulator when combined with Circle task, where agents are expected to learn circling behavior. Lower damping factors render the simulator unstable, necessitating this range.

Training is conducted on an Ubuntu 22.04 server using a Slurm job scheduler, which dynamically allocates computational resources. As resource allocations vary across runs, we do not report runtime comparisons for training.

\end{document}